\documentclass[nohyperref]{article}

\usepackage{microtype}
\usepackage{graphicx}
\usepackage{booktabs} 
\usepackage[T1]{fontenc}
\usepackage{amsfonts, amssymb, amsmath, amsthm, bbm, graphicx, color, tikz, booktabs, multirow, bm, cases, mathtools,  mathrsfs, pgfplots, subcaption, csvsimple,algorithm,algorithmic,wrapfig}
\usetikzlibrary{intersections, pgfplots.fillbetween, shapes, arrows, positioning, decorations.markings}
\definecolor {processblue}{cmyk}{0.96,0,0,0}
\tikzstyle{int}=[draw, fill=blue!20, minimum size=2em]
\tikzstyle{init} = [pin edge={to-,thin,black}]
\usepackage{pgfplotstable}
\usepackage{pgfplots}
\pgfplotsset{compat=1.14}
\usepackage{hyperref, graphics}
\hypersetup{colorlinks=true,linkcolor=blue,filecolor=gray, urlcolor=blue, citecolor=blue}
\usepackage{natbib}
\newtheorem{defn}{Definition}[section]
\newtheorem{lem}{Lemma}[section]
\newtheorem{rem}{Remark}[section]

\newtheorem{thm}{Theorem}[section]

\newtheorem{cor}{Corollary}[section]

\newcommand{\ex}[2]{{\ifx&#1& \mathbb{E} \else {\mathbb{E}_{#1}} \fi \left[#2\right]}}
\newcommand{\indp}{\perp\!\!\!\perp} 
\newcommand{\pr}[1]{\left(#1\right)}
\newcommand{\br}[1]{\left[#1\right]}
\newcommand{\abs}[1]{\left|#1\right|}
\newcommand{\kl}[2]{\mathrm{D_{KL}}\left(#1 ||#2 \right)}
\usepackage[toc,page,header]{appendix}
\usepackage{minitoc}

\doparttoc 
\faketableofcontents %
\usepackage{hyperref}

\usepackage[accepted]{icml2023}

\usepackage{amsmath}
\usepackage{amssymb}
\usepackage{mathtools}
\usepackage{amsthm}

\usepackage[capitalize,noabbrev]{cleveref}

\theoremstyle{plain}

\theoremstyle{definition}

\theoremstyle{remark}

\usepackage[textsize=tiny]{todonotes}

\icmltitlerunning{Tighter Information-Theoretic Generalization Bounds from Supersamples}

\begin{document}


\twocolumn[
\icmltitle{
Tighter Information-Theoretic Generalization Bounds from Supersamples
}



\icmlsetsymbol{equal}{*}

\begin{icmlauthorlist}
\icmlauthor{Ziqiao Wang}{yyy}
\icmlauthor{Yongyi Mao}{yyy}
\end{icmlauthorlist}

\icmlaffiliation{yyy}{Department of Electrical Engineering and Computer Science, University of Ottawa, Ottawa, Canada}

\icmlcorrespondingauthor{Ziqiao Wang}{zwang286@uottawa.ca}
\icmlcorrespondingauthor{Yongyi Mao}{ymao@uottawa.ca}

\icmlkeywords{Machine Learning, ICML}

\vskip 0.3in
]


\printAffiliationsAndNotice{}  

\begin{abstract}
In this work, we present a variety of novel information-theoretic generalization bounds for learning algorithms, from the supersample setting of \citet{steinke2020reasoning}\textemdash the setting of the ``conditional mutual information'' framework. Our development exploits projecting the loss pair (obtained from a training instance and a testing instance) down to a single number and  correlating loss values with a Rademacher sequence (and its shifted variants). The presented bounds include square-root bounds, fast-rate bounds, including those based on variance and sharpness, and bounds for interpolating algorithms etc. We show theoretically or empirically that these bounds are tighter than all information-theoretic bounds known to date on the same supersample setting.
\end{abstract}
\section{Introduction}
Using information-theoretic bounds to analyze the generalization properties of a learning algorithm has attracted increasing attention since the seminal works of \cite{russo2016controlling,russo2019much,xu2017information}. One major advantage of such bounds is that the information-theoretic quantities, e.g., the mutual information (MI) between the training sample and the trained parameter weights, are both distribution-dependent and algorithm-dependent. This makes them an ideal tool to characterize the generalization properties of a learning algorithm, particularly when the traditional algorithm-independent learning-theoretic tools (e.g., VC-dimension \cite{SLT98Vapnik} and Rademacher complexity \cite{bartlett2002rademacher}) appear inadequate. For example, \citet{zhang2017understanding,zhang2021understanding} show that the high-capacity deep neural networks can still generalize well,  contradicting the traditional wisdom in statistical learning theory that suggests
complex models tend to overfit the training data and perform poorly on unseen data \cite{SLT98Vapnik}. In contrast, the information-theoretic bounds have experimentally demonstrated that they are capable of 
tracking the generalization behaviour of modern neural networks \cite{negrea2019information,wang2021analyzing,harutyunyan2021informationtheoretic,wang2022generalization,wang2022two,hellstrom2022a}.

The original information-theoretic bound of \citet{xu2017information} has been extended or improved in many different ways, such as the chaining method \cite{asadi2018chaining,hafez2020conditioning,zhou2022stochastic,clerico22a}, the random subset or individual technique \cite{negrea2019information,bu2019tightening,haghifam2020sharpened,rodriguez2021random,zhou2022individually} and so on. 
Remarkably, \citet{steinke2020reasoning} has developed generalization bounds based on a conditional mutual information (CMI) measure obtained for a ``supersample'' setting. Specifically, the supersample is an $n\times 2$ matrix of data instances. In each row, one instance is selected at random for training and the other is masked out for testing. The authors then show that the CMI of the mask variables and the learned weights conditioned on the supersample can be used to upper-bound the generalization error. 
Although better behaving than the unconditional weight-based MI bounds (e.g., having boundedness guaranty), the CMI bounds can be difficult to measure for high-dimensional weights, which limits their application. To overcome such difficulty, functional CMI ($f$-CMI) bounds are proposed by \citet{harutyunyan2021informationtheoretic}, where the weight variable in CMI is replaced by the predictions for the supersample. In this case, each prediction pair is a two-dimensional discrete random variable, making the CMI easier to measure and also a tighter bound. 
More recently, \citet{hellstrom2022a} uses loss pairs to replace the predictions in $f$-CMI and obtain even tighter CMI bounds, known as evaluated CMI (e-CMI) bounds. In fact, the earliest version of e-CMI bound appeared in \citet{steinke2020reasoning}. The notion was also exploited in later works \cite{haghifam2021towards,haghifam2022understanding,haghifam2022limitations}. Note that e-CMI still measures the dependence between an one-dimensional variable (mask) and a two-dimensional variable (loss pair). In this work, we show that it is possible to further tighten the CMI bounds, using MI terms involving only two one-dimensional variables.

Our development is restricted to the supersample setting of \citet{steinke2020reasoning}, on which we establish novel CMI/MI bounds which are all easy to measure and tighter than the existing bounds in the same setting. Specifically, 
{\bf 1)} 
we first show that the loss pair used in e-CMI can be replaced by the loss difference, giving rise to a disintegrated CMI bound (Theorem~\ref{thm:bound-LD-cimi}) and an unconditional MI bound (Theorem~\ref{thm:bound-LD-ucimi}). Both are tighter than the previous square-root CMI bounds, all within the context of the same supersample construction. 
In particular, the obtained unconditional MI term can be interpreted as the \emph{achievable rate} over a memoryless channel in
communications.
We then show that in the interpolating regime (i.e., training error being zero) and under zero-one loss, the generalization error of the learning algorithm can be precisely expressed by the averaged communication rate (Theorem~\ref{thm:channel-interpolating}). In other words, we obtain the ``tightest bound'' of generalization error in this setting. We also establish a novel chained MI bound (Theorem~\ref{thm:bound-chaining-LD}) that is particularly advantageous for 
continuous and unbounded losses. 
{\bf 2)} 
Following a symmetric argument for Rademacher process, similar to \citet{zhivotovskiy2018localization}, we explicitly exploit the symmetric structure of expected generalization error by correlating losses with a Rademacher sequence and obtain a novel MI bound involving single losses (Theorem~\ref{bound-single-loss-sq}). 
Using the communication perspective, we show that the MI quantities in the bound are upper-bounded by the entropy function evaluated at half of the testing error
(Theorem~\ref{thm:channel-capacity-z}). 
{\bf 3)} 
By correlating losses with a shifted Rademacher sequence, we give novel fast-rate MI bounds of the weighted generalization error (Theorem~\ref{thm:bound-fast-rate-rademacher}). 
{\bf 4)} 
In order to enhance the fast-rate bound in the non-zero training error regime, we extend our analysis by deriving two additional bounds: a variance-based MI bound (Theorem~\ref{thm:bound-variance}) and a sharpness-based MI bound (Theorem~\ref{thm:bound-flatness}). These novel bounds also incorporate symmetric arguments, as shown in Lemma~\ref{lem:vaiance-symmetric} and Lemma~\ref{lem:flatness-symmetric}, respectively.
{\bf 5)} 
Experimental results show that our bounds nicely track the generalization dynamics of both linear models and non-linear neural networks, and our fast-rate bounds are tighter than the binary KL bound proposed in \citet{hellstrom2022a}, the tightest information-theoretic bound known to date for small, non-zero training error. 
{\bf 6)} 
As a by-product, we also develop a novel Wasserstein distance based bound (Theorem~\ref{thm:bound-LD-wass}). 

Proofs, additional analysis and experimental results are included in Appendix. 

\section{Preliminaries}
\subsection{Probability and Information Theory Notation}
Unless otherwise noted,  a random variable will be denoted by a capitalized letter, and  its realization by the corresponding lower-case letter. Let $P_X$ denote the distribution of a random variable $X$ and let $P_{X|Y}$ be the conditional distribution of $X$ conditioned on $Y$, which, upon conditioning on a specific realization, is denoted by $P_{X|Y=y}$ or simply $P_{X|y}$. 
Similarly, $\mathbb{E}_{X}$ is the expectation taken over $X\sim P_X$ and $\mathbb{E}_{X|Y=y}$ (or $\mathbb{E}_{X|y}$)is the expectation taken over  $X\sim P_{X|Y=y}$.
Let $H(\cdot)$ be the entropy and let $\mathrm{D_{KL}}(P||Q)$  denote the KL divergence of $P$ with respect to $Q$. Let $I(X;Y)$
be the mutual information (MI) between $X$ and $Y$, and $I(X;Y|Z)$ the conditional mutual information between $X$ and $Y$ conditioned on $Z$. Following \citep{negrea2019information}, we refer to $I^z(X;Y)\triangleq \mathrm{D_{KL}}(P_{X,Y|Z=z}||P_{X|Z=z}P_{Y|Z=z})$ as the disintegrated mutual information, and note that $I(X;Y|Z)=\ex{Z}{I^{Z}(X;Y)}$. Also, we use $\mathbb{W}(\cdot,\cdot)$ to denote the Wasserstein distance (formal definition is given in Appendix).
Throughout the paper, logarithm takes base $e$, making the unit of mutual information {\em nat}.

\subsection{Generalization Error}
We consider the supervised learning setting. Let $\mathcal{Z}=\mathcal{X}\times\mathcal{Y}$ be the domain of the instances, where  $\mathcal{X}$ and $\mathcal{Y}$ are input and label spaces respectively. A model of interest prescribes a family  
$\mathcal {F}$ of predictors, ${\mathcal F}\subset {\mathcal Y}^{\mathcal X}$, 
where each $f\in {\mathcal F}$ is parameterized by a vector $w$ in some space ${\cal W}$. We may write $f$ as $f_w$ as needed.
Let $\mu$ be the distribution of the instance and let $S=\{Z_i\}_{i=1}^n\overset{i.i.d}{\sim} \mu$ be the training sample. There is a learning algorithm $\mathcal{A}:\mathcal{Z}^n\rightarrow \mathcal{W}$, which takes the training sample $S$ as the input and outputs a hypothesis $W\in\mathcal{W}$, giving rise to a predictor $f_W\in \mathcal {F}$ that predicts label $Y$ for input $X$. Note that the algorithm $\mathcal {A}$ is characterized by a conditional distribution $P_{W|S}$. Suppose that the quality of the output hypothesis $W$ is evaluated by a loss function $\ell:\mathcal{W}\times\mathcal{Z}\rightarrow \mathbb{R}_0^+$. 
Then for a given $w$, we define the population risk $L_\mu(w)\triangleq \ex{Z'}{\ell(w,Z')}$, where $Z'\sim\mu$ is a testing instance. The quantity $L_\mu=\ex{W}{L_\mu(W)}$ is then the expected population risk. In practice, we cannot access the data distribution $\mu$, so we usually use the empirical risk as a proxy of the population risk, which is defined as $L_S(w)\triangleq \frac{1}{n}\sum_{i=1}^n\ell(w,Z_i)$ for a fixed $w$. Similarly, $L_n=\ex{W,S}{L_S(W)}$ is the expected empirical risk, where the expectation is taken over 
$P_{W,S}=\mu^n\otimes P_{W|S}$. Thus, $\mathrm{Err}\triangleq L_\mu-L_n$ is the expected generalization error.

\subsection{Supersample Setting}
The CMI framework for bounding generalization errors is first introduced in \citet{steinke2020reasoning}.
Let $\widetilde{Z}\in \mathcal{Z}^{n\times 2}$ be an $n\times 2$ matrix, serving as ``supersample'', where every entry is drawn i.i.d. from $\mu$. For notational convenience, we assume that the columns of $\widetilde{Z}$ are indexed by $\{0, 1\}$ instead of by $\{1, 2\}$. 
We further denote the $i$th row of $\widetilde{Z}$ as $\widetilde{Z}_i$ with entries $(\widetilde{Z}_{i,0},\widetilde{Z}_{i,1})$.
Let $U=(U_1, U_2, \ldots, U_n)^T\sim \mathrm{Unif}(\{0,1\}^n)$, independent of $\widetilde{Z}$, be used to select a training set $S$ from $\widetilde{Z}$:
$U_i=0$ dictates that $\widetilde{Z}_{i, 0}$ in $\widetilde{Z}$ be included in the training set $S$, and $\widetilde{Z}_{i, 1}$ be used for testing; $U_i$=1 dictates the opposite.  
Then, the constructed training sample $S$ is  equivalent to $\widetilde{Z}_U=\{\widetilde{Z}_{i,U_i}\}_{i=1}^n$. Let $\overline{U}_i=1-U_i$, then the testing sample is $\widetilde{Z}_{\overline{U}}=\{\widetilde{Z}_{i,\overline{U}_i}\}_{i=1}^n$. 
In addition, let $L_{i,0}\triangleq {\ell(\mathcal{A}(\widetilde{Z}_{U}),\widetilde{Z}_{i,0})}$ and $L_{i,1}$ defined similarly. We use $L_i=(L_{i,0},L_{i,1})$ to denote the loss pair in the $i$th row and $\Delta L_i=L_{i,1}-L_{i,0}$ be the difference in the pair. To avoid clutter, 
we might use the superscripts $+$ and $-$ to respectively replace the subscripts $0$ and $1$ in our notations, 
 e.g., $\widetilde{Z}^+_{i}=\widetilde{Z}_{i,0}$, $\widetilde{Z}^-_{i}=\widetilde{Z}_{i,1}$, $L^+_{i}=L_{i,0}$ and $L^-_{i}=L_{i,1}$.

\section{Generalization Bounds via Loss Difference}
\label{sec:loss-difference}

\subsection{Loss-Difference CMI Bound}
Using the loss difference, we first present the following square-root CMI bound.
\begin{thm}
    \label{thm:bound-LD-cimi}
    Assume that the loss is bounded between $[0,1]$, we have
    \[
    \left|\mathrm{Err}\right| \leq\frac{1}{n}\sum_{i=1}^n\mathbb{E}_{\widetilde{Z}}\sqrt{2I^{\widetilde{Z}}(\Delta L_i;U_i)}\leq\frac{1}{n}\sum_{i=1}^n\sqrt{2I(\Delta L_i;U_i|{\widetilde{Z}})}.
    \]
\end{thm}

Noting the Markov chain $U-W-f_W(\widetilde{Z}_i)-L_i-\Delta L_i$ (conditioned on $\widetilde{Z}$) and due to the data-processing inequality (DPI), this ``loss-difference CMI'' (or ``ld-CMI'') bound in Theorem~\ref{thm:bound-LD-cimi} (the second bound) is tighter than the bound in the previous works \citep{steinke2020reasoning,haghifam2020sharpened,harutyunyan2021informationtheoretic,hellstrom2022a}, namely,
    $
    \underbrace{I(\Delta L_i;U_i|{\widetilde{Z}})}_{\rm ld-CMI} \leq 
    \underbrace{I(L_i;U_i|{\widetilde{Z}})}_{\rm e-CMI}\leq \underbrace{I(f_W(\widetilde{Z}_i);U_i|{\widetilde{Z}})}_{f-{\rm CMI}} \leq \underbrace{I(W;U_i|{\widetilde{Z}})}_{\rm CMI}
    $. 
    It is remarkable that the ld-CMI bound can be significantly tighter. To see this, 
we re-express the loss function $\ell$ as a function on ${\cal Y}^2={\cal Y}\times {\cal Y}$, where $l(y, y')$ is  the loss value of the predicted label $y$ with respect to true label $y'$. We say that two elements $(y_1, y'_1)$ and $(y_2, y'_2)$ in ${\cal Y}^2$ are {\em loss-equivalent} and write 
$(y_1, y'_1) \equiv_\ell (y_2, y'_2)$ if $\ell(y_1, y'_1) = \ell(y_2, y'_2)$. It is straight-forward to verify that $\equiv_\ell$ is an equivalence relation on ${\cal Y}^2$. Let ${\cal L}$ denote the image of ${\cal Y}^2$ under $\ell$.
The quotient space ${\cal Y}^2/\equiv_\ell$, or the set of equivalence classes modulo $\equiv_\ell$, has a one-to-one correspondence with ${\cal L}$, under which we may identify ${\cal Y}^2/\equiv_\ell$ with ${\cal L}$. Furthermore, we say that two loss pairs
$(\ell_A, \ell'_A)$ and $(\ell_B, \ell'_B)$ in ${\cal L}^2={\cal L}\times {\cal L}$ are {\em loss-difference-equivalent} and write 
$(\ell_A, \ell'_A) \equiv_\Delta (\ell_B, \ell'_B)$ if $\ell_A - \ell'_A = \ell_B- \ell'_B$. Then $\equiv_\Delta$ is likewise an equivalence relation on ${\cal L}^2$, which induces the quotient space ${\cal L}^2/\equiv_\Delta$. Note that $f_W(\widetilde{Z}_i)$ is a random variable on ${\cal Y}^4={\cal Y}^2\times {\cal Y}^2$ whereas $\Delta L_i$ is a essentially a random variable on ${\cal L}^2/\equiv_\Delta$, which can be identified with $\left({\cal Y}^2/\equiv_\ell\right)^2/\equiv_\Delta$ under the aforementioned one-to-one correspondence. There can be a significant reduction of space size from ${\cal Y}^4$ to $\left({\cal Y}^2/\equiv_\ell\right)^2/\equiv_\Delta$ when ${\cal Y}$ or ${\cal L}$ is large (assuming they are finite, to fix ideas). Thus, $\Delta L_i$ reveals much less information about $U_i$ than
$f_w(\widetilde{Z}_i)$ does, making the term $I(\Delta L_i;U_i|{\widetilde{Z}})$ significantly smaller than $I(f_W(\widetilde{Z}_i);U_i|{\widetilde{Z}})$ and suggesting that 
the ld-CMI bound 
can be much tighter than the $f$-CMI bound. A similar argument can be made comparing the ld-CMI and the e-CMI bounds.



It is noteworthy that the loss-difference CMI bound 
is easier to compute than the $f$-CMI and e-CMI bounds, 
since $\Delta L_i$ is a scalar. Interestingly, when regarding $\Delta L_i$ as a (scaled) one-dimensional projection of $L_i$ on a particular direction, the term $I(\Delta L_i;U_i|{\widetilde{Z}})$ shares some similarity with the notion of \emph{Sliced Mutual Information} (SMI) recently proposed in \citep{goldfeld2021sliced,goldfeld2022ksliced}; the difference is that SMI requires averaging over a random direction of projection.

\begin{figure}[htbp]
\centering
\tikzstyle{signal}=[draw=none,fill=none]
\tikzstyle{spre}=[semithick, <-]
\tikzstyle{tpre}=[thick, <-]
\tikzstyle{vpre}=[very thick, <-]
\tikzstyle{upre}=[ultra thick, <-]

\begin{subfigure}{0.23\textwidth}
\centering
\begin{tikzpicture}[bend angle=45]
  \node[signal] (u0) {$0$};
  \node[signal] (u1) [below of=u0, node distance=1.6cm] {$1$};
  \node[signal, right of=u0, node distance=3cm] (l1) {$1$};
\draw [semithick, ->] (u0) -- (l1) node[font =\scriptsize, midway,above] {$1-\alpha_i-\epsilon_i$};
\draw [semithick, ->] (u1) -- (l1) node[font =\scriptsize, pos=.2,above] {$\epsilon_i$};
  \node[signal, below of=l1, node distance=0.8cm] (l0) {$0$};
  \draw [semithick, ->] (u0) -- (l0) node[font =\scriptsize, midway,above] {$\alpha_i$};
\draw [semithick, ->] (u1) -- (l0) node[font =\scriptsize, midway,below] {$\alpha_i$};
  \node[signal, right of=u1, node distance=3cm] (l2) {$-1$};
  \draw [semithick, ->] (u1) -- (l2) node[font =\scriptsize, midway,below] {$1-\alpha_i-\epsilon_i$};
\draw [semithick, ->] (u0) -- (l2) node[font =\scriptsize, pos=.2,below] {$\epsilon_i$};
\end{tikzpicture}
\label{fig:bs-be}
\end{subfigure}
\begin{subfigure}{0.23\textwidth}
\centering
\begin{tikzpicture}[bend angle=45]
  \node[signal] (u0) {$0$};
  \node[signal] (u1) [below of=u0, node distance=1.6cm] {$1$};
  \node[signal, right of=u0, node distance=3cm] (l1) {$0$};
\draw [semithick, ->] (u0) -- (l1) node[font =\scriptsize, midway,above] {$1-p_i$};
\draw [semithick, ->] (u1) -- (l1) node[font =\scriptsize, pos=.3,left] {$q_i$};
  \node[signal, right of=u1, node distance=3cm] (l2) {$1$};
  \draw [semithick, ->] (u1) -- (l2) node[font =\scriptsize, midway,below] {$1-q_i$};
\draw [semithick, ->] (u0) -- (l2) node[font =\scriptsize, pos=.3,left] {$p_i$};
\end{tikzpicture}
\label{fig:ba-z}
\end{subfigure}
\caption{Left: channel from $U_i$ to $\Delta L_i$. Right: channel from $U_i$ to $L^+_i$. Zero-one loss assumed.
}
\label{fig:binarychannel}
\end{figure}
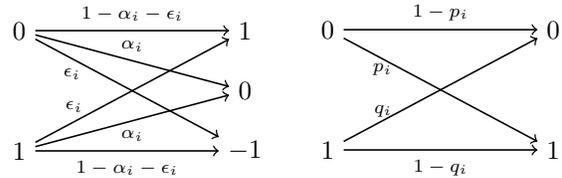
\subsection{Loss-Difference MI Bound}

Under the setting of supersample as above, we can also obtain a generalization bound based on the  loss-difference MI without conditioning on the supersample.  
\begin{thm}
    \label{thm:bound-LD-ucimi}
    Assume that $\ell(\cdot,\cdot)\in[0,1]$, then
    \[
    \left|\mathrm{Err}\right| \leq\frac{1}{n}\sum_{i=1}^n\sqrt{2I(\Delta L_i;U_i)}.
    \]
\end{thm}
    By the independence of $U_i$ and $\widetilde{Z}$,  $I(\Delta L_i;U_i)\leq I(\Delta L_i;U_i)+I(U_i;\widetilde{Z}|\Delta L_i)=I(\Delta L_i;U_i|\widetilde{Z})$. Then the bound in Theorem~\ref{thm:bound-LD-ucimi} is tighter than 
    ld-CMI bound in Theorem~\ref{thm:bound-LD-cimi}, although not directly comparable to the first bound in Theorem~\ref{thm:bound-LD-cimi}.



It is interesting to relate the MI $I(\Delta L_i; U_i)$ to a communication setting where $P_{\Delta L_i|U_i}$ specifies a memoryless channel with input $U_i$ and output $\Delta L_i$. Then $I(\Delta L_i; U_i)$ is the {\em rate of reliable communication} over this channel achievable with the input distribution $P_{U_i}$ (which is ${\rm Bern}(\frac{1}{2})$ by the construction of $U$) \cite{shannon1948mathematical}.  Consider the special case where $\ell(\cdot, \cdot)$ is the {\em zero-one loss}, i.e., 
$\ell(w,z)= \mathbbm{1}_{f_w(x)\neq y}$. In this case, $\Delta L_i\in\{-1,0,1\}$, and the channel is shown in Figure~\ref{fig:binarychannel} (left), in which $\epsilon_i$ and $\alpha_i$ are transition probabilities as shown on the respective transition edges. In particular, recalling $\Delta L_i=L^-_{i}-L^+_{i}$, we see that $\alpha_i$ is the probability that in $\widetilde{Z}_i$ the instance selected from training has the same loss value as that selected for testing, and that $\epsilon_i$ is the probability that the training instance in $\widetilde{Z}_i$ has a higher loss value than the testing instance. It follows that 
any {\em interpolating algorithm}, namely, one that achieves zero training error must have $\epsilon_i=0$ for each $i$. The following theorem can then be proved.

\begin{thm}
    \label{thm:channel-interpolating} Under zero-one loss and for any interpolating algorithm ${\cal A}$,  $I(\Delta L_i;U_i)=(1-\alpha_i)\ln{2}$ nats for each $i$, and $\abs{\mathrm{Err}}=L_\mu=\sum_{i=1}^n\frac{I(\Delta L_i;U_i)}{n\ln{2}}$.
\end{thm}

    In this case, the generalization error is exactly determined by the communication rate over the channel in Figure~\ref{fig:binarychannel} (left) averaged over all such channels, making Theorem~\ref{thm:channel-interpolating} the obviously the ``tightest bound'' of generalization error in the ``interpolating regime''. It is of course also tighter than the interpolating bound in \citet{hellstrom2022a}, which may be alternatively seen from 
    $I(\Delta L_i;U_i)\leq I(L_i;U_i|\widetilde{Z})$.
    Note that \citet{haghifam2022understanding} also gives a MI quantity that can determine the generalization error in the interpolating case, although their leave-one-out MI is between an $n+1$-dimensional random variable and an one-dimensional random variable, and its corresponding bound is
    established without exploiting the communication perspective.


Furthermore, it is possible to establish further tightened loss-difference MI bounds for more general loss functions than those required in Theorem \ref{thm:bound-LD-ucimi}. Specifically, the loss function can be unbounded and continuous, as presented in next theorem, where we apply 
the chaining technique \cite{asadi2018chaining,hafez2020conditioning,zhou2022stochastic,clerico22a} and the obtained bound consists of   
MI terms between $U_i$ and the successively quantized versions of $\Delta L_i$. 
To that end,  let $\mathrm{Err}^i(\Delta \ell_i)\triangleq (-1)^{U_i}\Delta \ell_i$ and let $\Gamma\subseteq \mathbb{R}$ be the range of $\Delta \ell$. Then $\{\mathrm{Err}^i(\Delta \ell_i)\}_{\Delta \ell_i\in\Gamma}$ is a random process\footnote{\label{fn-chain}Some prerequisite definitions of the chaining technique (such as \emph{stochastic chain}, \emph{separable process} and \emph{sub-Gaussian process})
     are give in the Appendix~\ref{sec:defns and lemmas}.}, applying the \emph{stochastic chaining} method \cite{zhou2022stochastic} gives the following chained MI bound.

\begin{thm}
    \label{thm:bound-chaining-LD}
    For each $i\in [n]$, we assume $\{\Delta L_{i,k}\}_{k=k_0}^\infty$ is a stochastic chain\textsuperscript{\ref{fn-chain}}  of 
    $(\{\mathrm{Err}^i(\Delta \ell_i)\}_{\Delta \ell_i\in\Gamma},\Delta L_i)$, then
    \[
    \mathrm{Err}\leq \frac{1}{n}\sum_{i=1}^n\sum_{k=k_0}^\infty\sqrt{2\ex{}{
    |\Delta L_{i,k}-\Delta L_{i,k-1}|^2
    }I(\Delta L_{i,k};U_i)},
    \]
    where $\Delta L_{i,k}$ is the $k$th level of quantization of $\Delta L_{i}$, the RHS expectation is taken over $(\Delta L_{i,k},\Delta L_{i,k-1})$. 
\end{thm}
    Notice that the bound is expressed as MI terms each involving
    $U_i$ and $\Delta L_{i, k}$, both being discrete random variables. 
    This has not arose in the previous chained weight-based MI bounds where they either contain the continuous random variable $S$ \cite{asadi2018chaining,zhou2022stochastic,clerico22a} or are conditioned on the continuous random variable $\widetilde{Z}$ \cite{hafez2020conditioning}. Additionally, by the master definition of MI \citep[Eq.(8.54)]{thomas2006elements}, we know that $I(\Delta L_i;U_i)=\sup_k I(\Delta L_{i,k};U_i)$,
    and $I(\Delta L_{i,k};U_i)\to I(\Delta L_i;U_i)$ when $k\to \infty$.

For bounded loss,  the diameter $\mathrm{diam}(\Gamma)$ is finite, we can use hierarchical partitions as in \citet{asadi2018chaining} to construct a deterministic sequence of $\{\Delta L_{i,k}\}_{k=k_0}^\infty$. This is deferred to Corollary~\ref{cor:bound-hierarchical-chain} in Appendix.


\subsection{Loss-Difference Bound Beyond CMI and MI}

It is possible to develop generalization bounds based on the loss differences in the supersample using distances or divergences beyond the information-theoretic measures. Here we present such a bound based on  Wasserstein distance. As investigated in the previous literature \citep{rodriguez2021tighter}, Wasserstein distance usually gives a tighter bounds than the mutual information.

\begin{thm}
    \label{thm:bound-LD-wass}
    Assume that $\ell(\cdot,\cdot)\in[0,1]$, then 
    \[
    \left|\mathrm{Err}\right| \leq\frac{1}{n}\sum_{i=1}^n\ex{U_i}{\mathbb{W}(P_{\Delta L_i|U_i},P_{\Delta L_i})}.
    \]
\end{thm}
Unlike the results in \citet{rodriguez2021tighter}, here we do not require the loss to be Lipschitz continuous.




\section{Generalization Bounds via Correlating with Rademacher Sequence}
\label{sec:single-loss}
 We have so far obtained tighter square-root MI bounds based on the information measures (and their variants) between the loss difference $\Delta L_i$ and the mask variable $U_i$. However, the loss difference may not be used to obtain the fast-rate generalization bound where the square root function is removed \citep{grunwald2021pac,hellstrom2021fast,hellstrom2022a}. This is because deriving the fast-rate bound usually relies on a weighted generalization error, for which one loses the center-symmetric structure of the standard generalization error. Specifically, knowing $\Delta L_i$ and $U_i$ is sufficient to determine the generalization error at $i$th position by $(-1)^{U_i}\Delta L_i$. However, for the weighted generalization error at $i$th row defined by $E^i_{C_1}=L_{i,\overline{U}_i}-C_1L_{i,{U}_i}$ (for some constant $C_1>0$), having $U_i$ and a weighted loss difference $\Delta_{C_1} L_i=L_i^--C_1L_i^+$, does not allow its recovering from $(-1)^{U_i}\Delta_{C_1} L_i$ since  $L_i^+-C_1L_i^-\neq C_1L_i^+-L_i^-$ in general. Indeed, knowing both $L_i^--C_1L_i^+$ and $ L_i^+-C_1L_i^-$ requires knowing $L_i$. Then in order to obtain fast-rate bounds, we need to give up the loss difference and return to the original e-CMI as in \citet{hellstrom2022a}.

Therefore, if we still want to use a MI between two one-dimensional random variables to bound the error,  we need to find another trick. This motivates us to use a Radamecher viewpoint to derive the CMI bounds.
Before we handle the fast-rate CMI bound for the weighted generalization error, we again consider the standard generalization error.

\subsection{Single-Loss MI Bounds}
Although the CMI setting, particularly its construction of the ``ghost sample'', is conceptually related to the Rademacher complexity \citep{bartlett2002rademacher}, the information-theoretic generalization bounds in previous literature do not explicitly exploit this connection.
Fortunately, both information-theoretic bounds \cite{negrea2019information,hellstrom2020generalization,hellstrom2021fast,hellstrom2022a} and the Rademacher viewpoint \cite{kakade2008complexity,yang2019fast} are shown connected to the PAC-Bayes bounds, we thus derive a variant of e-CMI bound by invoking a similar symmetric argument with \cite{zhivotovskiy2018localization,yang2019fast}.

We first note the following lemma.
\begin{lem}
    \label{lem:standard-symmetric}
    The expected generalization error $\mathrm{Err}=\frac{2}{n}\sum_{i=1}^n\ex{L^+_{i},\varepsilon_i}{\varepsilon_iL^+_{i}}$,
    where $\varepsilon_i = (-1)^{\overline{U}_i}$. 
\end{lem}

Note that $\varepsilon_{1:n}=\{\varepsilon_i\}_{i=1}^n$ is a sequence of Rademacher random variables, and the lemma suggests that
$
\mathrm{Err}=
2\mathbb{E}_{\widetilde{Z}^+_{1:n}}\mathbb{E}_{\varepsilon_{1:n}}\ex{L^+_{1:n}|\varepsilon_{1:n},\widetilde{Z}^+_{1:n}}{\frac{1}{n}\sum_{i=1}^n \varepsilon_iL^+_{i}},
$
where $\widetilde{Z}^+_{1:n}=\{\widetilde{Z}^+_{i}\}_{i=1}^n$ and $L^+_{1:n}=\{L^+_{i}\}_{i=1}^n$.
Then, recall that the Rademacher complexity is defined as 
$
\mathfrak{R}_n(\mathcal{W})\triangleq \mathbb{E}_{S}\ex{\varepsilon_{1:n}}{\sup_{w\in \mathcal{W}}{\frac{1}{n}\sum_{i=1}^n \varepsilon_i\ell(w,Z_i)}}
$ \cite{bartlett2002rademacher}.
Notably, the expected generalization error can be viewed, up to a scale factor 2, as an ``average'' version of the Rademacher complexity. While $\mathrm{Err}$ considers the average correlation between the loss sequence and the Rademacher sequence, the Rademacher complexity measures the worst such correlation. Thus, $\mathrm{Err}\leq 2\mathfrak{R}_n(\mathcal{W})$.

Based on Lemma~\ref{lem:standard-symmetric}, we have the following bound. 
\begin{thm}
    \label{bound-single-loss-sq}
    Assume $\ell(\cdot)\in[0,1]$, we have
    \[
    \left|\mathrm{Err}\right| \leq\frac{2}{n}\sum_{i=1}^n\sqrt{2I(L^+_{i};U_i)}\leq \frac{2}{n}\sum_{i=1}^n\sqrt{2I(f_W(X^+_{i});U_i|\widetilde{Z})}.
    \]
\end{thm}
The variable $U_i$ in the above MI/CMI terms can obviously be replaced by $\varepsilon_i$. Thus the theorem can be interpreted as using a different notion of
``average correlation'', namely mutual information, between losses (or predictions) and Rademacher noises to bound the original notion of average correlation (as stated in Lemma  \ref{lem:standard-symmetric} and  discussed earlier).


This bound may not be directly comparable to others due to the undesired constant of $2$ outside of the square root function in the bound. We will soon see that $I(L^+_{i};U_i)$ based bound will be more useful when the square root is removed.

For the zero-one loss, 
the dependence between $U_i$ and $L^+_{i}$ is characterized by the communication channel given in Figure~\ref{fig:binarychannel} (right). In this case, 
$U_i=0$ indicates $\widetilde{Z}^+_{i}$ is selected for training, then $p_i$ is the error rate on this training instance. Similarly, when $U_i=1$,  $\widetilde{Z}^+_{i}$ is used for testing, then $1-q_i$ is the error rate on this testing instance. In practice, we usually have $p_i<1-p_i$ since $L^+_{i}$ is more likely to be zero when $\widetilde{Z}^+_{i}$ is a training instance, and 
we may also have $p_i<1-q_i$  since $L^+_{i}$ is more likely to be one when $\widetilde{Z}^+_{i}$ is a testing instance compared with the case when $\widetilde{Z}^+_{i}$ is used in training. When $p_i=0$, this channel reduces to a $Z$-channel \cite{thomas2006elements}. This corresponds to an interpolating algorithm, 
for which we have the following theorem.
\begin{thm}
    \label{thm:channel-capacity-z}
    For zero-one loss and any interpolating algorithm, we have $\frac{1}{n}\sum_{i=1}^nI(L^+_{i};U_i)\leq H(\frac{L_\mu}{2})$. 
\end{thm}

When the loss is not discrete, we can again obtain a chained MI bound by quantizing the continuous random variable $L_i^+$, which is given in Theorem~\ref{thm:chaining-single-loss-bound} in Appendix~\ref{sec:chianed-single-loss}. 

\subsection{Fast-Rate MI Bound}
We are now in a position to discuss the weighted generalization error, ${\mathrm{Err}}_{C_1}\triangleq L_\mu-(1+C_1)L_n$, where $C_1$ is a prescribed constant. This notion is important for obtaining the fast-rate PAC-Bayes bounds \cite{catoni2007pac}. 

To bound this weighted generalization error, similar to Lemma~\ref{lem:standard-symmetric}, we have the following symmetry argument.
\begin{lem}
    \label{lem:weighted-symmetric}
    The weighted generalization error can be rewritten as 
    \[
    \mathrm{Err}_{C_1}=\frac{2+C_1}{n}\sum_{i=1}^n\ex{L^+_{i},\tilde{\varepsilon}_i}{ \tilde{\varepsilon}_iL^+_{i}},
    \]
    where $\tilde{\varepsilon}_i=(-1)^{\overline{U}_i}
    -\frac{C_1}{C_1+2}$ is a \emph{shifted} Rademacher variable with mean $-\frac{C_1}{C_1+2}$.
\end{lem}
The relationship between ${\rm Err}$ and Rademacher complexity also likewise extends to that between $\mathrm{Err}_{C_1}$ and ``shifted Rademacher complexity" defined as
    $\widetilde{\mathfrak{R}}_n(\mathcal{W})\triangleq \mathbb{E}_{S}\ex{\tilde{\varepsilon}_{1:n}}{\sup_{w\in \mathcal{W}}{\frac{1}{n}\sum_{i=1}^n \tilde{\varepsilon}_i\ell(w,Z_i)}}$, namely $\mathrm{Err}_{C_1}\leq 2\widetilde{\mathfrak{R}}_n(\mathcal{W})$.

Then, we are ready to present the following bounds.
\begin{thm}
    \label{thm:bound-fast-rate-rademacher}
    Let $\ell(\cdot,\cdot)\in[0,1]$. There exist $C_1,C_2>0$ 
    such that
    \begin{align}    
         L_\mu\leq& (1+C_1)L_n+ \sum_{i=1}^n \frac{I(L^+_{i};U_i)}{C_2n},\label{ineq:bound-fast-rate-general}\\
         L_\mu \leq& L_n+\sum_{i=1}^n \frac{4I(L^+_{i};U_i)}{n}+4\sqrt{\sum_{i=1}^n \frac{L_nI(L^+_{i};U_i)}{n}}. \label{ineq:bound-fast-rate-general-optimal}
    \end{align}
    Furthermore, if $\mathcal{A}$ is an interpolating algorithm, we have
    \begin{align}
    \label{ineq:bound-fast-rate-interpolate}
    L_\mu\leq \sum_{i=1}^n \frac{2I(L^+_{i};U_i)}{n\ln{2}}. 
    \end{align}
\end{thm}

    Notice that Eq.~(\ref{ineq:bound-fast-rate-general-optimal}) does not depend on $C_1,C_2$, as 
    it is obtained via minimizing the bound in Eq~(\ref{ineq:bound-fast-rate-general}) 
    over a region of $(C_1,C_2)$ in which Eq~(\ref{ineq:bound-fast-rate-general}) hold.
    
    Comparing Eq.~(\ref{ineq:bound-fast-rate-interpolate}) with the interpolating bound in \citet[Eq.~(12)]{hellstrom2022a}, the main difference is that their bounds\footnote{Note that \citet{hellstrom2022a} uses $I(L^+_{i},L^-_{i};U_i|\widetilde{Z})$ but this CMI term can be strengthened to the unconditional MI by using the same development in this paper.} are based on $I(L^+_{i},L^-_{i};U_i)$,  instead of $2I(L^+_{i};U_i)$. This difference could be characterized by the \emph{interaction information} \cite{yeung1991new}, namely $I(L_i^+;U_i;L_i^-)=I(L_i^+;U_i)-I(L_i^+;U_i|L_i^-)=2I(L^+_{i};U_i)- I(L_i;U_i)$ (where the second equality is by the chain rule of MI), and the  value $I(L^+_{i};L^-_{i};U_i)$ could be positive, negative and zero.
Hence, the interpolating bound could be further improved as below
\begin{align}
    \label{ineq:bound-optimal-interpolate}
    L_\mu\leq \sum_{i=1}^n \frac{\min\{2I(L^+_{i};U_i), I(L_i;U_i)\}}{n\ln{2}}. 
    \end{align}
    This bound is strictly non-vacuous since
    the RHS of Eq.~(\ref{ineq:bound-fast-rate-general-optimal}) is upper-bounded by $\frac{\sum_{i=1}^nH(U_i)}{n\ln{2}}=1$. Note that the ``tightest bound'' of the interpolating algorithm is already obtained in Theorem~\ref{thm:channel-interpolating}.
Previous works \cite{steinke2020reasoning,hellstrom2022a} suggest that the fast-rate bounds for the weighted generalization error are typically useful when the empirical risk is small or even zero, which may restrict their applications. In the sequel, we introduce two new types of MI bound that can further extend Eq.~(\ref{ineq:bound-fast-rate-general}) in Theorem~\ref{thm:bound-fast-rate-rademacher}.  


\subsection{Variance Based MI Bound}
Inspired by the above Rademacher perspective, we first present a new bound that depends on the MI term and a notion of loss variance, defined below.
\begin{defn}[$\gamma$-Variance]
    For any $\gamma\in (0,1)$, $\gamma$-variance for a learning algorithm is defined as
    \[
    V(\gamma)\triangleq\ex{W,S}{\frac{1}{n}\sum_{i=1}^n\pr{\ell(W,Z_i)-(1+\gamma)L_S(W)}^2}.
    \]
\end{defn}
By definition, $\gamma$-variance also depends on the data distribution. In the zero-one loss case, it can be characterized by the following lemma.
\begin{lem}
    \label{lem:empirical-variance}
    Under the zero-one loss assumption, we have $V(\gamma)=L_n-(1-\gamma^2)\ex{W,S}{L^2_S(W)}$.
\end{lem}
Loss variances, of any kind, have not appeared in the information-theoretic bounds
developed to date.  Such a notion however does arise in the PAC-Bayes literature, where such an idea traces back to \cite{seldin2012pac,tolstikhin2013pac}. Different from these works, here we utilize an expected empirical variance, and the distribution of $W$ in this case is generated by the learning algorithm rather than the posterior distribution used for prediction in PAC-Bayes.

The gap between $\mathrm{Err}$ and $V(\gamma)$ also has a ``symmetry lemma'' (similar to Lemma~\ref{lem:weighted-symmetric}) correlating to the shifted Rademacher sequence. 
\begin{lem}
\label{lem:vaiance-symmetric}
For any $C_1>0$, we have 
\[
     \mathrm{Err}-C_1V(\gamma)\leq\frac{2+C_1\gamma^2}{n}\sum_{i=1}^n\ex{L^+_{i},\tilde{\varepsilon}_i}{ \tilde{\varepsilon}_iL^+_{i}},
     \]
     where $\tilde{\varepsilon}_i=\varepsilon_i-\frac{C_1\gamma^2}{C_1\gamma^2+2}$ is the shifted Rademacher variable with mean $-\frac{C_1\gamma^2}{C_1\gamma^2+2}$.
\end{lem}



\begin{thm}
    \label{thm:bound-variance}
    Assume $\ell(\cdot,\cdot)\in\{0,1\}$, $\gamma\in(0,1)$. Then, there exist $C_1,C_2>0$ 
    such that
    \begin{align}
        \mathrm{Err}\leq &C_1V(\gamma)+\sum_{i=1}^n\frac{I(L_i^+;U_i)}{nC_2}.\label{ineq:variance-bound-general}
    \end{align}
\end{thm}
    Notably, the interpolating setting is a sufficient but not necessary condition for the zero $\gamma$-variance, that is, $L_n=0$ makes $V(\gamma)=0$, but $V(\gamma)=0$ does not indicate that $L_n=0$. In addition, by Lemma~\ref{lem:empirical-variance}, Eq.~(\ref{ineq:variance-bound-general}) can be rewritten as $L_\mu\leq (1+C_1)L_n-C_1(1-\gamma^2)\ex{W,S}{L^2_S(W)}+ \sum_{i=1}^n \frac{I(L^+_{i};U_i)}{C_2n}$ so for the fixed $C_1$ and $C_2$, the bound of Eq.~(\ref{ineq:variance-bound-general}) is tighter than the bound of Eq.~(\ref{ineq:bound-fast-rate-general}) with the gap being at least $C_1(1-\gamma^2)\ex{W,S}{L^2_S(W)}$.

\subsection{Sharpness Based MI Bound}
The nice generalization property of deep neural networks is often credited to the ``flat minima'' \cite{jastrzkebski2017three} of loss landscapes. Recently, \citet{neu2021information} and \citet{wang2022generalization} have proved that the generalization error can be upper-bounded by a MI based term plus a sharpness (or flatness) related term. 
Following the similar development in the previous section, we are able to obtain a bound that also depends on a MI term and a sharpness term, where we use a completely different analysis with \citep{neu2021information,wang2022generalization}. 

We first define a notion of sharpness.
\begin{defn}[$\lambda$-Sharpness]
    For any $\lambda\in (0,1)$, the ``$\lambda$-sharpness''  at position $i$ of the training set is defined as 
    \[
    F_i(\lambda)\triangleq\ex{W,Z_i}{\ell(W,Z_i)-(1+\lambda)\mathbb{E}_{W|Z_i}{\ell(W,Z_i)}}^2.
    \]
\end{defn}
This $\lambda$-sharpness can be regarded as an expected version of the ``flatness'' used in \citet{yang2019fast} with $W\sim P_{W|Z_i}$ instead of some posterior distribution of $W$.
\begin{lem}
\label{lem:flatness-upper-bound}
    Assume $\ell(\cdot)\in\{0,1\}$, we have  $F_i(\lambda)=\ex{W,Z_i}{\ell(W,Z_i)}-(1-\lambda^2)\ex{Z_i}{\mathbb{E}^2_{W|Z_i}{\ell(W,Z_i)}}$.
\end{lem}

Let $F(\lambda)=\frac{1}{n}\sum_{i=1}^nF_i(\lambda)$. 
Similar to Lemma~\ref{lem:standard-symmetric},  Lemma~\ref{lem:weighted-symmetric} and Lemma~\ref{lem:vaiance-symmetric}, we have the following symmetric argument.
\begin{lem}
    \label{lem:flatness-symmetric}
     For any $C_1>0$, we have 
     \begin{align*}
         \mathrm{Err}-&C_1F(\lambda)=\\
         &\frac{C_1+2}{n}\sum_{i=1}^n\ex{L^+_{i},U_i}{\tilde{\varepsilon}_iL^+_{i}-\frac{C_1(1-\lambda^2)}{C_1+2}\hat{\varepsilon}_ih(U_i)},
     \end{align*}
    where $\tilde{\varepsilon}_i=\varepsilon_i-\frac{C_1}{C_1+2}$ and $\hat{\varepsilon}_i=\varepsilon_i-1$ are the shifted Rademacher variables, and $h(U_i)=\ex{\widetilde{Z}^+_{i}|U_i}{\mathbb{E}^2_{L^+_{i}|\widetilde{Z}^+_{i},U_i}L^+_{i}}$.
\end{lem}

We are then ready to present the following bound.
\begin{thm}
    \label{thm:bound-flatness}
    Assume $\ell(\cdot,\cdot)\in\{0,1\}$, $\lambda\in(0,1)$. Then, there exist $C_1,C_2>0$ such that
    \begin{align}
        \mathrm{Err}\leq&  C_1F(\lambda)+\sum_{i=1}^n\frac{I(L^+_{i};U_i)}{C_2n}\label{ineq:flatness-bound-general}.
    \end{align}
\end{thm}

Similar to the variance based bound, zero $\lambda$-sharpness is a weaker condition than the interpolating assumption. In particular, 
    Eq.~(\ref{ineq:flatness-bound-general}) could be tighter than Eq.~(\ref{ineq:bound-fast-rate-general}) in Theorem~\ref{thm:bound-fast-rate-rademacher} when the empirical risk is non-zero. Specifically, Eq.~(\ref{ineq:flatness-bound-general}) can be rewritten as $L_\mu\leq (1+C_1)L_n-  C_1(1-\lambda^2){L_n^2}+\sum_{i=1}^n\frac{I(L^+_{i};U_i)}{C_2n}$ by Lemma~\ref{lem:flatness-upper-bound} and Jensen's inequality. If $C_1,C_2$ are fixed, then the sharpness based bound is always tighter than Eq.~(\ref{ineq:bound-fast-rate-general}) and the gap is at least $C_1(1-\lambda^2){L_n^2}$.

\begin{figure*}[!ht]
    \centering
    \begin{subfigure}[b]{0.24\textwidth}
\includegraphics[scale=0.28]{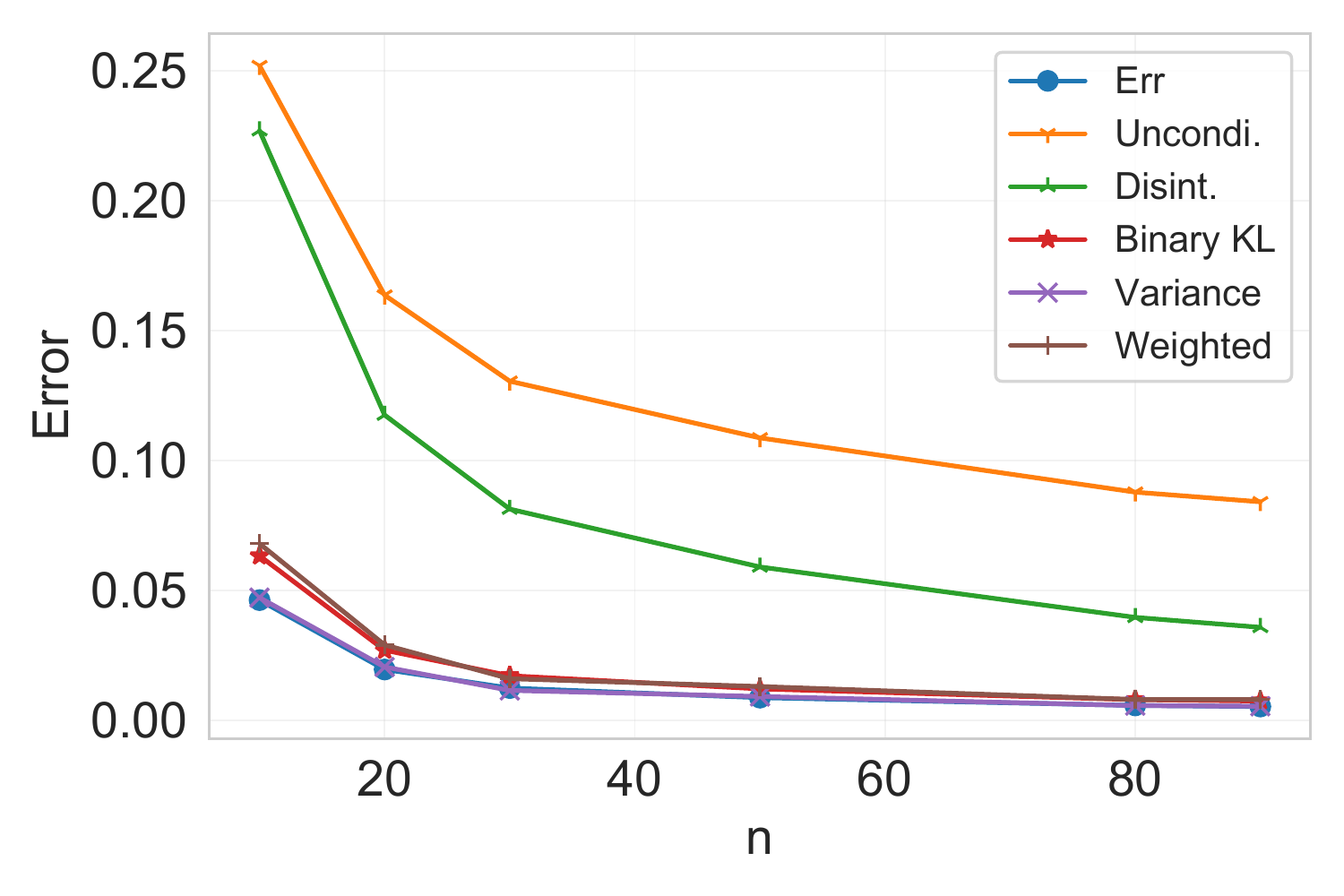}            \caption{$|\mathcal{Y}|=2$ (Realizable)}            \label{fig:binary-easy}
    \end{subfigure}
\begin{subfigure}[b]{0.24\textwidth}
\includegraphics[scale=0.28]{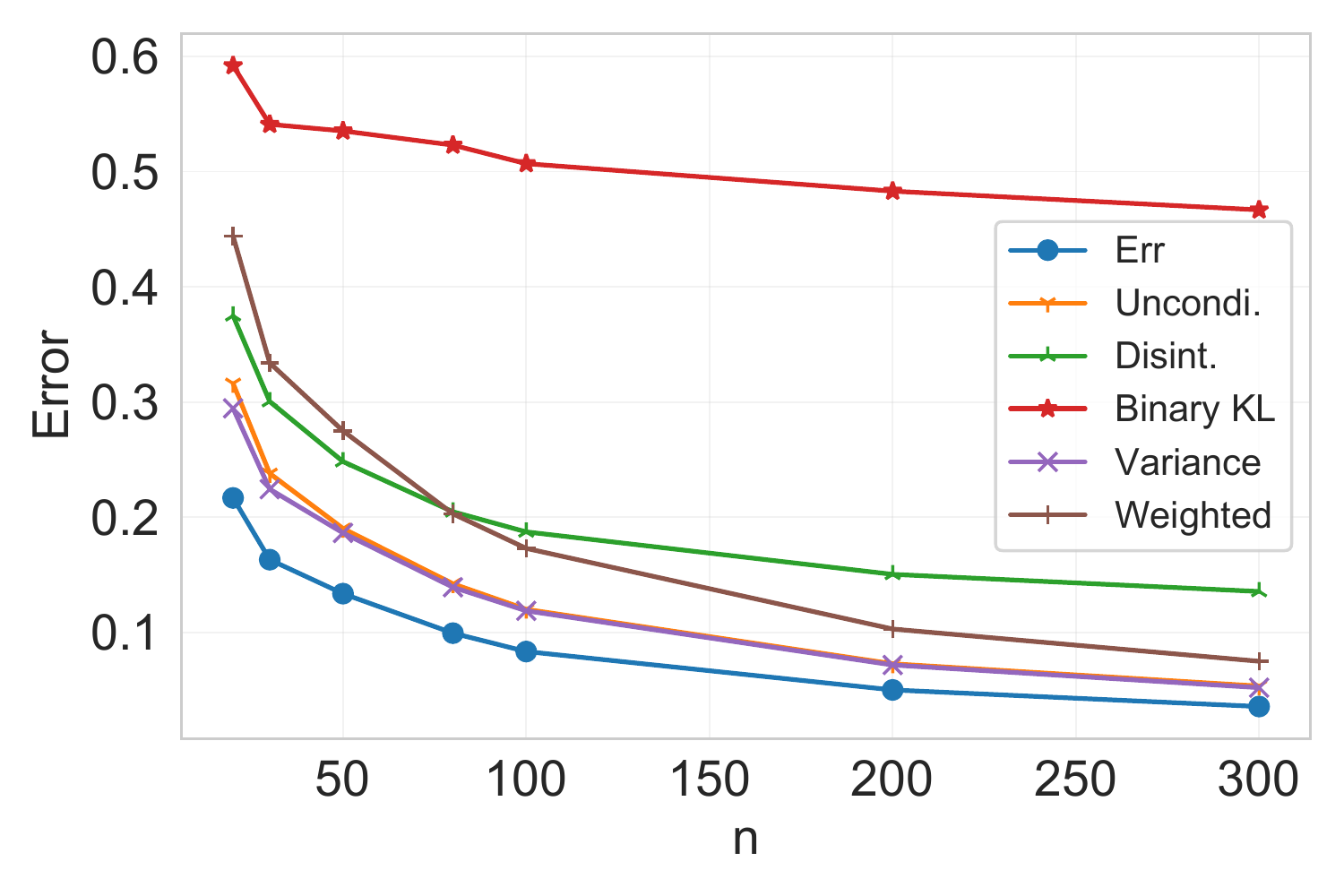}
\caption{$|\mathcal{Y}|=2$ (Non-Separable)}
    \label{fig:binary-hard}
\end{subfigure}
 \begin{subfigure}[b]{0.24\textwidth}
\includegraphics[scale=0.28]{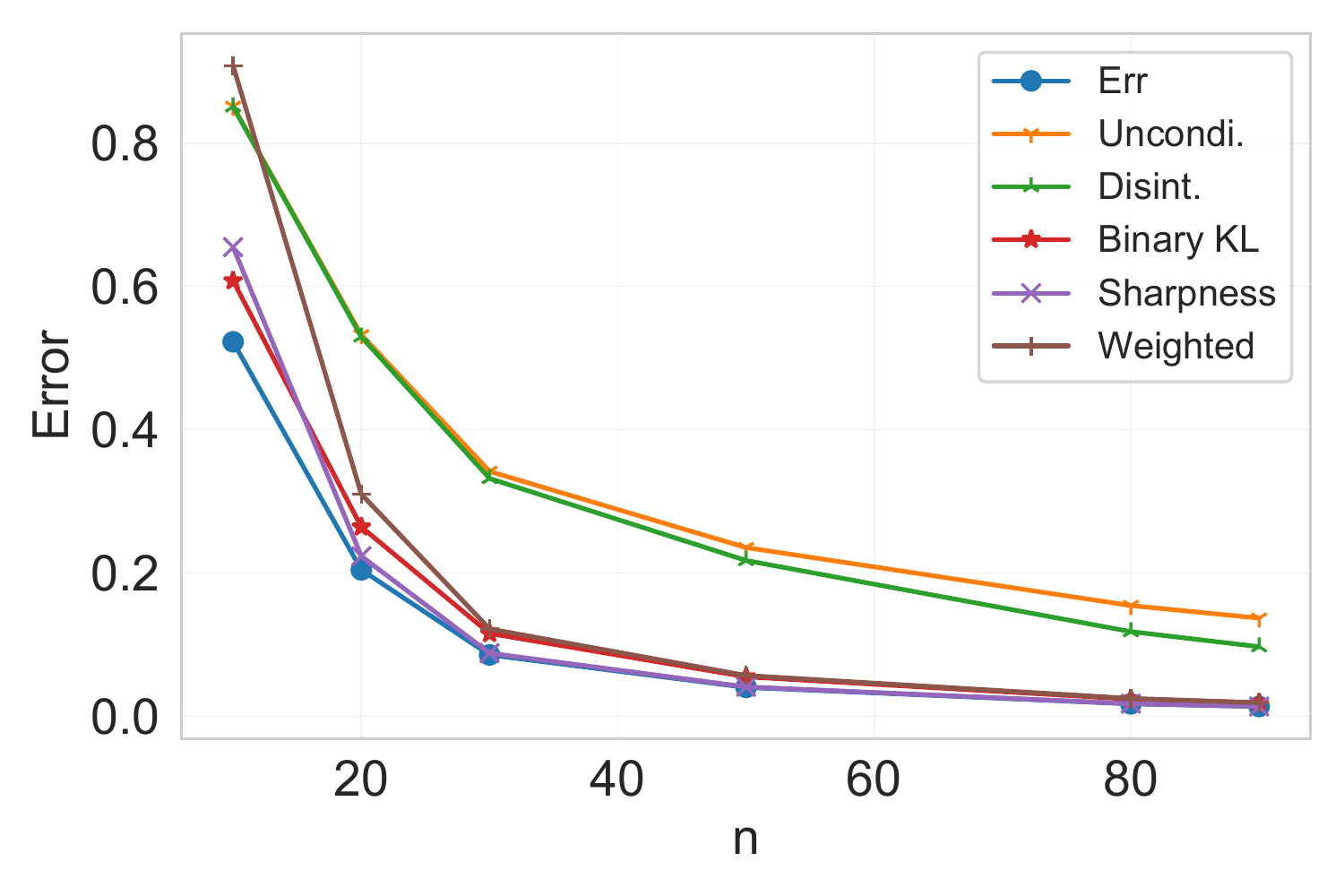}
\caption{$|\mathcal{Y}|=10$ (Realizable)}
\label{fig:ten-easy}
    \end{subfigure}
\begin{subfigure}[b]{0.24\textwidth}
\includegraphics[scale=0.28]{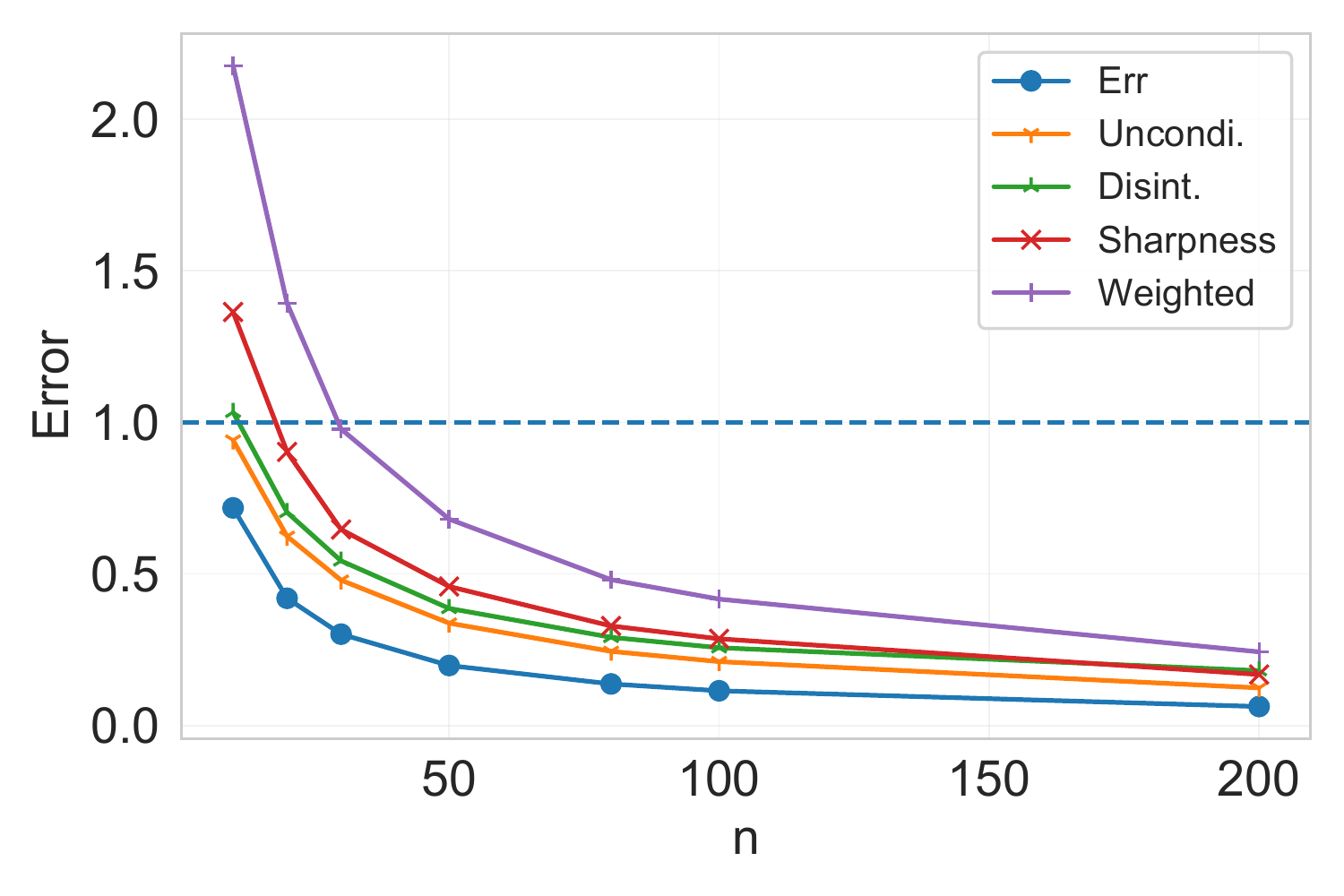}
\caption{$|\mathcal{Y}|=10$ (Non-Separable)}
\label{fig:ten-hard}
\end{subfigure}
\caption{Comparison of bounds on the synthetic dataset. 
(a) Binary classification with a separable $\mu$ (i.e. the interpolating setting). Notice that the variance bound nearly coincides with $\mathrm{Err}$. (b) Binary classification with a non-separable $\mu$. (c) Ten-class classification with a separable $\mu$. (d) Ten-class classification with a non-separable $\mu$. The binary KL bound is removed in (d) since it is always $\geq 1$.}\label{fig:Gaussian-data}
\end{figure*}

To conclude this section, we give a bound that combines the variance and the sharpness.
\begin{cor}
    \label{cor:bound-variance-flatness}
    Assume $\ell(\cdot,\cdot)\in\{0,1\}$ and $\gamma,\lambda\in(0,1)$, then there exist $C_1,C_2>0$ such that
    \[
        \mathrm{Err}\leq C_1\min\{V(\gamma),F(\lambda)\}+\sum_{i=1}^n\frac{I(L_i^+;U_i)}{nC_2}.
    \]
\end{cor}
    We remark that if $\mathcal{A}$ satisfies any of the following: (i) $L_n\to 0$; (ii)$V(\gamma)\to 0$ for some $\gamma\in(0,1)$; (iii)$F(\lambda)\to 0$ for some $\lambda\in(0,1)$, we all have
    $
        \mathrm{Err}\leq \sum_{i=1}^n\frac{2I(L_i^+;U_i)}{n\ln{2}}.
    $
\begin{figure*}[!ht]
    \centering
    \begin{subfigure}[b]{0.24\textwidth}
\includegraphics[scale=0.28]{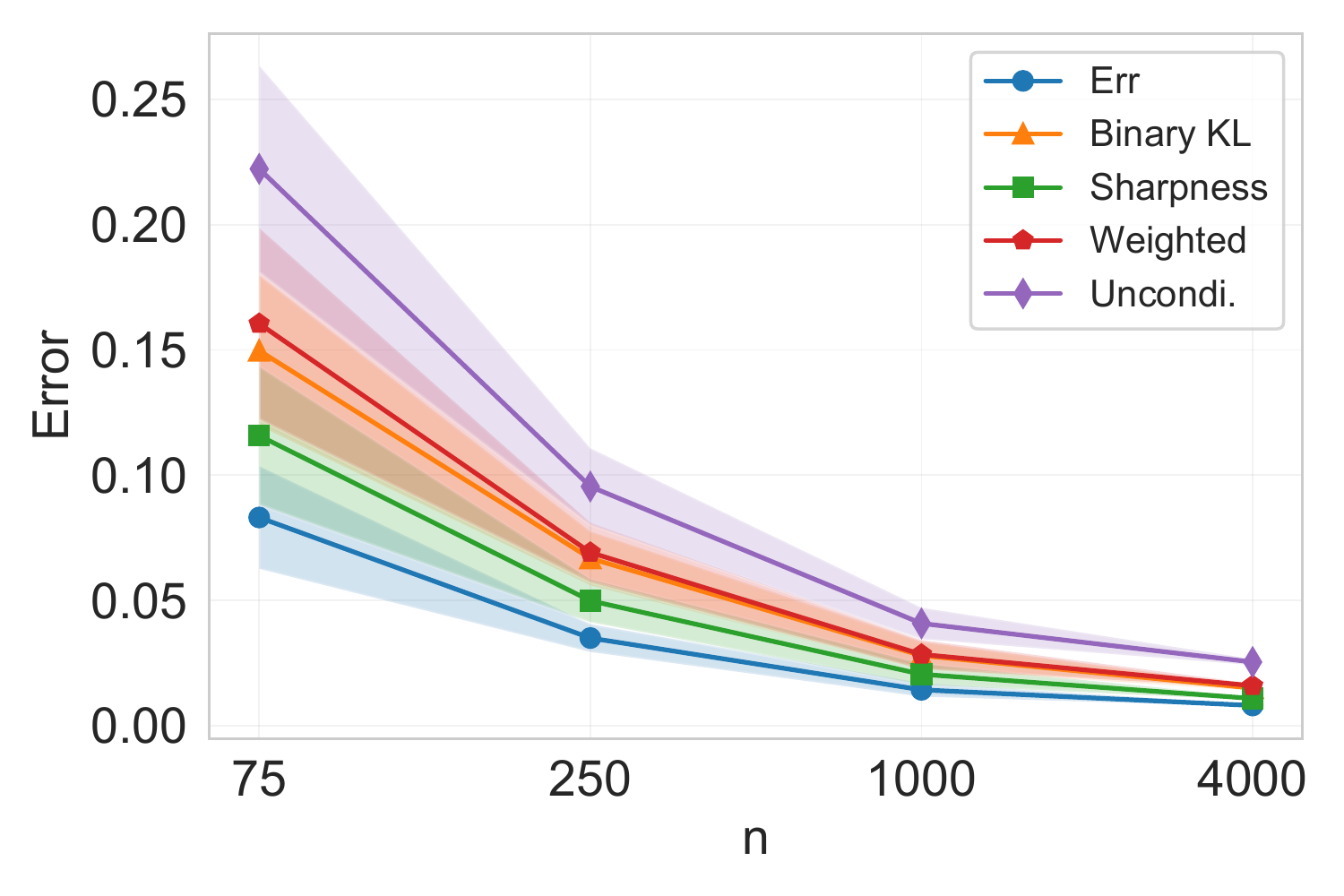}            
\caption{CNN on MNIST}            \label{fig:cnn-mnist}
    \end{subfigure}
\begin{subfigure}[b]{0.24\textwidth}
\includegraphics[scale=0.28]{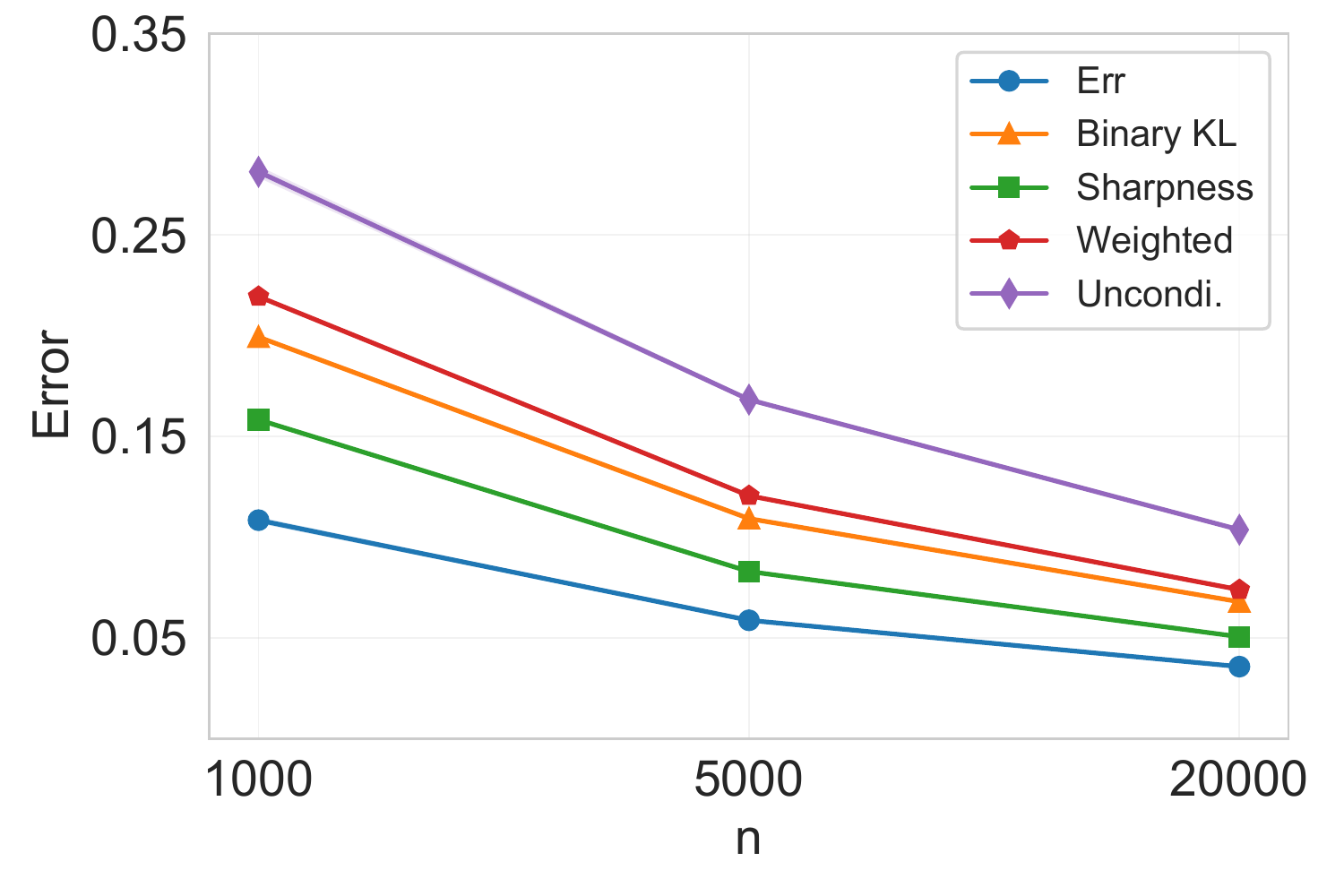}
\caption{ResNet on CIFAR10}
    \label{fig:resnet-cifar}
\end{subfigure}
 \begin{subfigure}[b]{0.24\textwidth}
\includegraphics[scale=0.28]{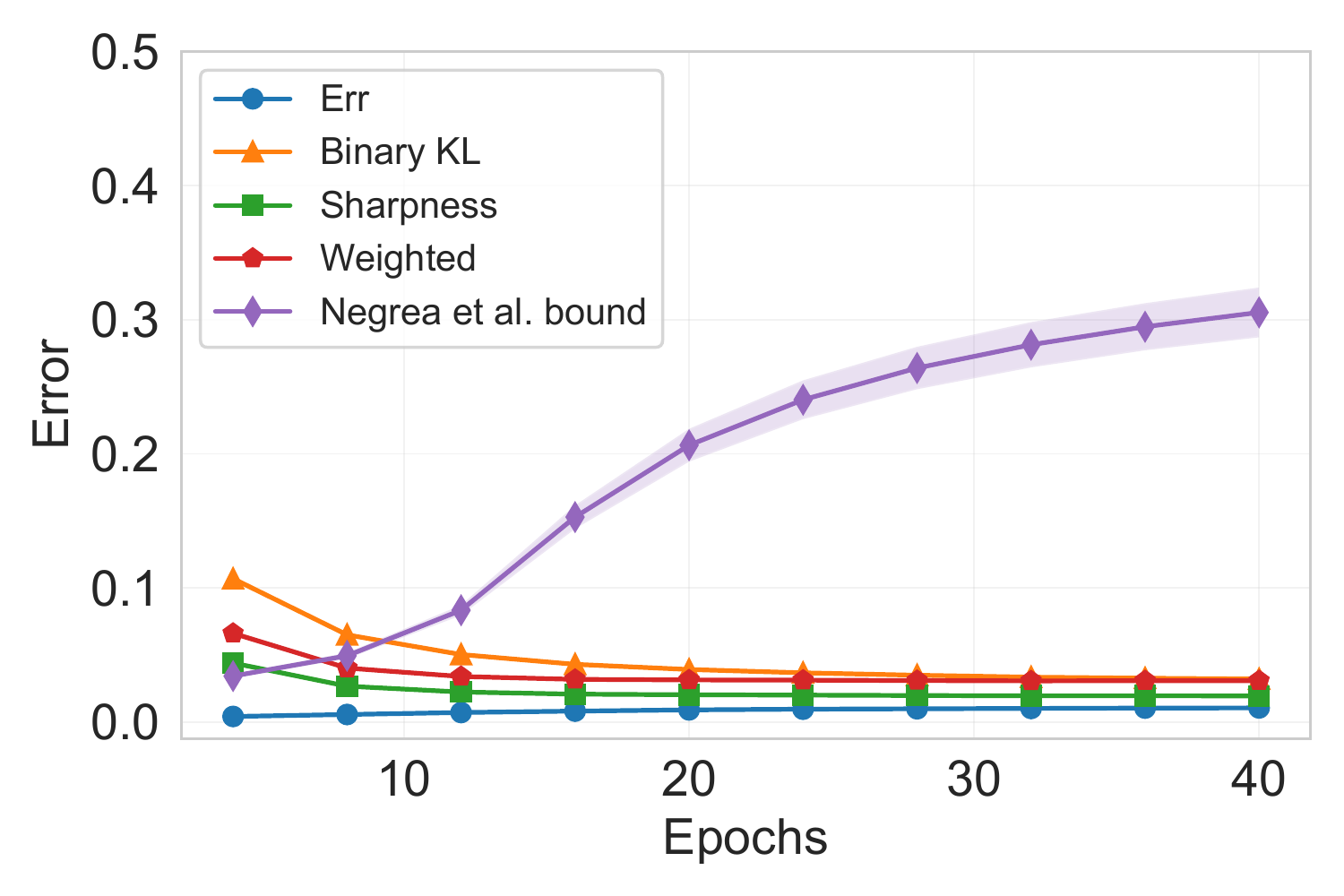}
\caption{SGLD (MNIST)}
\label{fig:sgld-mnist}
    \end{subfigure}
\begin{subfigure}[b]{0.24\textwidth}
\includegraphics[scale=0.28]{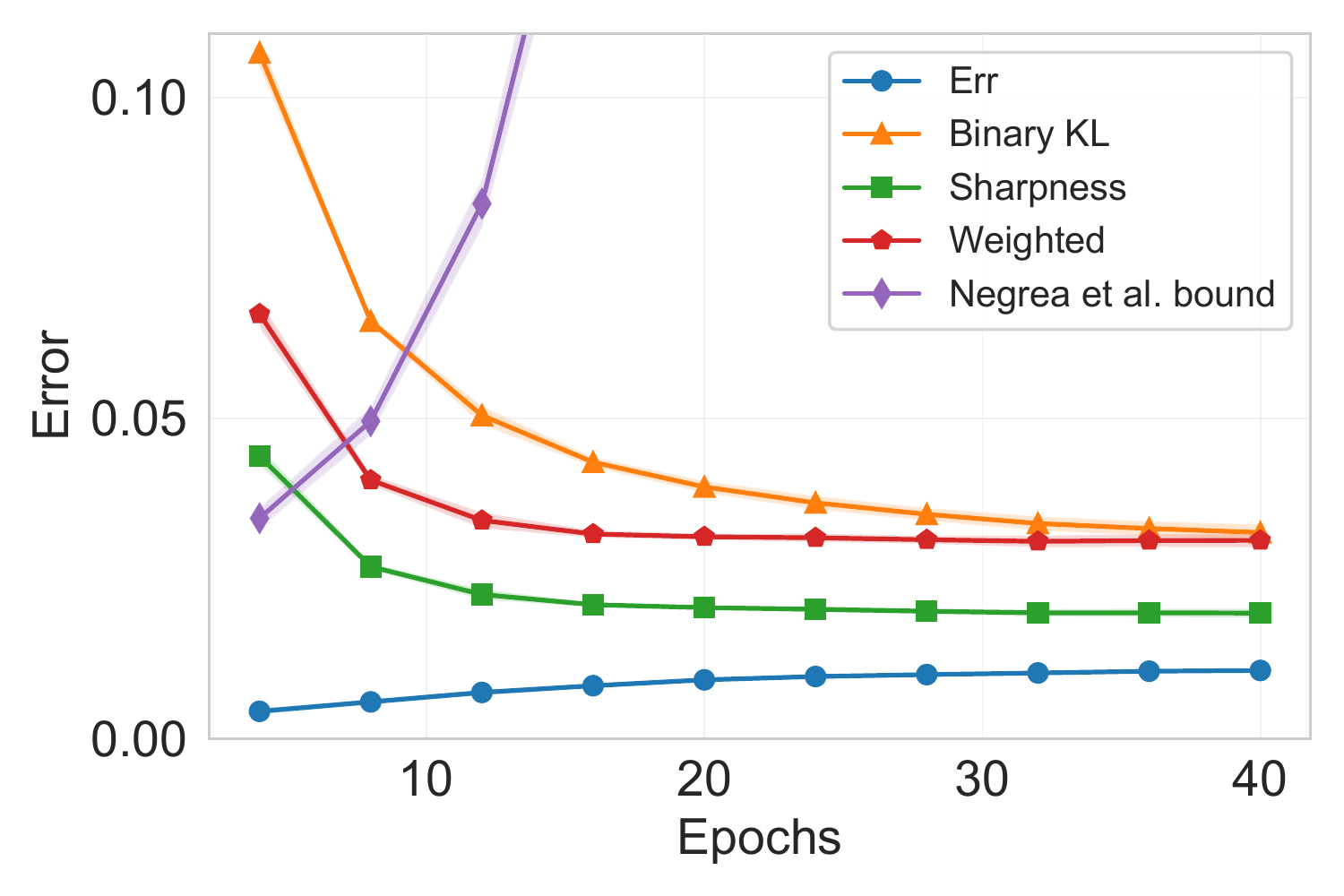}
\caption{Zoomed-in of (c)}
\label{fig:zoom-in}
\end{subfigure}
\caption{Comparison of bounds on two real datasets, MNIST (``$4$ vs $9$'') and CIFAR10.}\label{fig:Real-data}
\end{figure*}

\section{Numerical Results}
\label{sec:compare}
In this section, we empirically compare some CMI and MI bounds discussed in our paper. Our first experiment is based on a synthetic Gaussian dataset, where a simple linear classifier (with a softmax output layer) will be trained. The second experiment follows the same deep learning scenario setting with \cite{harutyunyan2021informationtheoretic,hellstrom2022a}, where we will train a $4$-layer CNN on MNIST \cite{lecun2010mnist} and fine-tune a ResNet-50 \cite{he2016deep} (pretrained on ImageNet \cite{deng2009imagenet}) on CIFAR10 \cite{krizhevsky2009learning}. In all of these experiments, we let the loss be the zero-one loss, namely $\ell(w,z)= \mathbbm{1}_{f_w(x)\neq y}$, and we apply the empirical risk minimization (ERM) to find the hypothesis, namely $w=\arg\min_{w\in\mathcal{W}} L_S(w)$. Since such loss is not differentiable, to enable the gradient based optimization methods such as SGD, we hereby use  the cross-entropy loss as a surrogate classification loss during training. Notice that $\mathrm{Err}$ is still defined with respect to the zero-one loss in our experiments. 

Under these settings, we will mainly compare the disintegrated ld-CMI bound in the first inequality of Theorem~\ref{thm:bound-LD-cimi} (\emph{Disint.}), the unconditional MI bound in Theorem~\ref{thm:bound-LD-ucimi} (\emph{Uncondi.}), the weighted generalization error bound in Eq.~(\ref{ineq:bound-fast-rate-general-optimal}) of Theorem~\ref{thm:bound-fast-rate-rademacher} (\emph{Weighted}), the variance bound in Theorem~\ref{thm:bound-variance} (\emph{Variance}) and the sharpness bound in Theorem~\ref{thm:bound-flatness} (\emph{Sharpness}). Besides, we will include the binary KL bound in \citet{hellstrom2022a} as a baseline, which is, to the best of our knowledge, the tightest fast-rate CMI bound in the literature when $L_n$ is close (but not equal) to zero. In addition, we note that the difference between the variance bound and the sharpness bound is negligible in the current scale of the figures, so for each figure we only report one of them. The comparison between the variance bound and the sharpness bound, and more comparison of other bounds mentioned in this paper (such as interpolating bounds and single-loss based square-root bounds) are given in Appendix~\ref{sec:experiments-appendix}.


\subsection{Linear Classifier}
We will first use a simple linear classifier to carry out the Gaussian data classification task (see Appendix~\ref{sec:setup} for more details of data generation and training). There are at least two major benefits of using such a synthetic dataset. On the one hand, the ground-truth distribution $\mu$ is known so we can draw unlimited supersamples, allowing repeatition of experiments so as to obtain an accurate estimate of the desired quantity
(e.g., for each $n$, we repeat the experiment $5000$ times, each with a random $(\widetilde{Z}, U)$ and report the average). On the other hand, the separability of different classes is adjustable, allowing for a control of the task difficulty.
Specifically, we will consider both the zero training loss case (i.e. a separable $\mu$) and the high training loss case (i.e. a non-separable $\mu$).

In the binary classification tasks (i.e. $|\mathcal{Y}|=2$),
the evaluations of $\mathrm{Err}$ and the bounds are given in Figure~\ref{fig:binary-easy} and Figure~\ref{fig:binary-hard}. When $\mu$ is separable (Figure~\ref{fig:binary-easy}), the algorithm can always interpolate the training sample. In this case, the fast-rate bounds are tighter than the square-root bounds, and the variance bound (or the sharpness bound) is the tightest. Moreover, notice that here the the disintegrated CMI bound is tighter than the unconditional MI bound. For a more challenging classification task (Figure~\ref{fig:binary-hard}), $L_n$ is no longer zero, the square-root bounds become tighter than the binary KL bound. Indeed, \citet{hellstrom2022a} shows that when the empirical risk is large, the square-root bound will be tighter than their fast-rate bounds. In contrast, our variance bound is even slightly tighter than the square-root bound of Theorem~\ref{thm:bound-LD-ucimi} in Figure~\ref{fig:binary-hard}. Additionally, notice that unlike the realizable case (the one with separable $\mu$), the unconditional MI bound is now tighter than the disintegrated CMI bound. 

 We also conduct experiments in the ten-class classification task (i.e. $|\mathcal{Y}|=10$), and the results are shown in Figure~\ref{fig:ten-easy} and Figure~\ref{fig:ten-hard}. In the realizable case (Figure~\ref{fig:ten-easy}), the results are similar to binary classification except that the binary KL bound is tighter than all the other bounds when $n=10$, which is the only case we observe where the binary KL bound outperforms Theorem~\ref{thm:bound-variance} and Theorem~\ref{thm:bound-flatness}. In addition, it is worth mentioning that Eq.~(\ref{ineq:bound-fast-rate-general-optimal}) decays much faster than the square-root bounds in Figure~\ref{fig:ten-easy} (and also in Figure~\ref{fig:binary-hard}). For the non-separable case in Figure~\ref{fig:ten-hard}, only the unconditional MI bound in
 Theorem~\ref{thm:bound-LD-ucimi} is non-vacuous for all the values of $n$. While the binary KL bound is removed in this case since it is always vacuous, our sharpness bound is  competitive to the square-root bound when $n\geq50$.

\subsection{Neural Networks}
To compare information-theoretic generalization bounds of modern deep neural networks, we follow the same experiment settings in
\cite{harutyunyan2021informationtheoretic,hellstrom2022a}. Specifically, we train a 4-layer CNN model on a binary MNIST dataset (``$4$ vs $9$'') and also fine-tune a pretrained ResNet-50 on CIFAR10. Unlike the previous synthetic dataset case, here we can only repeatedly run experiments (with different $\widetilde{Z}$ and $U$) for limited times due to the high computation complexity. Thus, we report the both average numerical values and their standard deviations. Notice that
our code is primarily the same as the code provided by \citet{hellstrom2022a}, which is originally based on the code in \href{https://github.com/hrayrhar/f-CMI}{https://github.com/hrayrhar/f-CMI}. 
 More experiment details can be found in Appendix~\ref{sec:setup}.

Observations in the binary MNIST experiment (Figure~\ref{fig:cnn-mnist}) are close to the realizable binary classification case in Figure~\ref{fig:binary-easy} (both near the interpolating regime). For example, the fast-rate bounds are tighter than the square-root bound. Notably, our sharpness bound (or variance bound)  significantly improve the the binary KL bound in both the MNIST experiment (Figure~\ref{fig:cnn-mnist}) and the CIFAR10 experiment (Figure~\ref{fig:resnet-cifar}), while Eq.~(\ref{ineq:bound-fast-rate-general-optimal}) is slightly weaker than the binary KL bound.
Furthermore, we also compare the bounds when the CNN model is trained by a SGLD algorithm \cite{raginsky2017non}, a variant of SGD, on the binary MNIST dataset. In this case, we add the weight-based MI bound of SGLD in \citet[Eq.~6]{negrea2019information} as a baseline. Figure~\ref{fig:sgld-mnist} suggests that both of our sharpness bound and Eq.~(\ref{ineq:bound-fast-rate-general-optimal}) improve the binary KL bound. Notably, \citet{harutyunyan2021informationtheoretic} and \citet{hellstrom2022a} observe that $f$-CMI bound and e-CMI bound are worse than the SGLD bound in \citet{negrea2019information} at the beginning of training. As shown in Figure~\ref{fig:zoom-in}, although our sharpness bound is still looser than the SGLD bound in \citet{negrea2019information} before the fifth epoch, our sharpness bound significantly shrinks the gap with the SGLD bound during early training.

\section{Limitations and Other Related Literature}

\paragraph{Limitations}
More recently, the limitations of information-theoretic bounds in explaining the generalization properties of stochastic convex optimization (SCO) problems have been investigated by \citet{haghifam2022limitations}. In their study, the authors demonstrate that almost all existing information-theoretic bounds, except for the chained MI/CMI bounds, fail to vanish in at least one of their counterexamples, despite the true generalization error vanishing. Unfortunately, neither our loss-difference  MI/CMI bounds nor our single-loss MI bounds are capable of overcoming such limitations revealed in their constructed counterexample presented in \citep[Theorem~17]{haghifam2022limitations}. These limitations shed light on certain inherent properties of mutual information measures, which may not be easily overcome solely by introducing new information measures.

In \citet{hellstrom2022a}, the authors provide an e-CMI generalization bound for a generic convex function of the training loss and test loss. Although our analysis, using either loss-difference CMI/MI bounds or single-loss MI bounds, may not be directly applicable to general convex comparison functions between the training loss and testing loss, one potential alternative is to consider the square of the loss difference, for which similar techniques can be employed to derive generalization bounds.

Furthermore, it is important to note that all our new information-theoretic generalization bounds are derived under the assumption of independent and identically distributed (i.i.d.) training instances. Exploring the possibility of relaxing this assumption represents a promising avenue for future research.

\paragraph{Other Related Work} Information-theoretic generalization bounds have been explored for some specific algorithms. For example, the weight based information-theoretic bounds have been successfully applied to characterize the generalization properties of SGLD \cite{pensia2018generalization,bu2019tightening,negrea2019information,haghifam2020sharpened,rodriguez2021random,wang2021analyzing}, and more recently, these bounds are also used to analyze either the vanilla SGD \cite{neu2021information,wang2022generalization} or the stochastic differential equations (SDEs) approximation of SGD \cite{wang2022two}. To apply the weight based MI or CMI bounds for SGD and its variants, unlike the bounds in our paper and \cite{harutyunyan2021informationtheoretic,hellstrom2022a} that treat the learning algorithm as a black-box, these weight based bounds are usually further upper bounded by some quantities along the trajectories of the algorithms (e.g.,  gradient incoherence \cite{negrea2019information}). This then points to a future direction: Can the losses-based or predictions-based information-theoretic bounds be exploited the similar trajectory information of the gradient based algorithms?

It is also noteworthy that recently
\citet{haghifam2022understanding} and \citet{rammal2022on} concurrently propose a variant of the initial CMI framework \cite{steinke2020reasoning}, the ``leave-one-out'' (LOO) CMI setting, where their supersample is a $n+1$ vector instead of a $n\times 2$ matrix. While our development in this paper is restricted to the  latter, 
it is curious and tempting to compare\textemdash or connect\textemdash the two.

 In addition to the supervised learning setting, information-theoretic bounds have found applicability in various other learning scenarios, showcasing their versatility. These scenarios include meta-learning \cite{hellstrom2022evaluated}, semi-supervised learning \cite{he2022information}, transfer learning \cite{wu2020information}, domain adaptation \cite{wang2023informationtheoretic}, and so on. It is reasonable to expect that our findings can be effectively utilized in these diverse learning settings as well.

\section{Concluding Remarks}
In this work, we obtain some novel and easy-to-measure information-theoretic generalization bounds. These bounds are demonstrated to be tighter than the previous results in the same supersample setting of \citet{steinke2020reasoning}, either theoretically or empirically. In our development, we also discuss some other viewpoints of generalization in the current supersample construction including explaining generalization via the rate of reliable communication over the memoryless channel, and via correlating with the Rademacher sequence. These new insights may help to design new learning algorithms or discover novel algorithm-dependent complexity measures. 


\section*{Acknowledgements}
This work is supported partly by an NSERC Discovery grant and a National Research Council of Canada (NRC) Collaborative R\&D grant (AI4D-CORE-07). Ziqiao Wang is also supported in part by the NSERC CREATE program through the Interdisciplinary Math and Artificial Intelligence (INTER-MATH-AI) project.
The authors would like to thank the anonymous reviewers for their careful reading and valuable suggestions.

\bibliography{ref}
\bibliographystyle{icml2023}

\newpage

\onecolumn

\begin{appendices}

\section{Some Useful Definitions and Lemmas}
\label{sec:defns and lemmas}

\begin{defn}[Wasserstein Distance]
Let $d(\cdot,\cdot)$ be a metric and let $P$ and $Q$ be probability measures on $\mathcal{X}$. Denote  $\Gamma(P,Q)$ as the set of all couplings of $P$ and $Q$ (i.e. the set of all joint distributions  on $\mathcal {X} \times \mathcal {X}$ with two marginals being $P$ and $Q$), then the Wasserstein Distance of order one between $P$ and $Q$ is defined as $\mathbb{W}(P,Q)\triangleq\inf_{\gamma\in\Gamma(P,Q)}\int_{\mathcal{X}\times\mathcal{X}} d(x,x')d\gamma(x,x')$. 
\end{defn}

Definition~\ref{defn:separable-process}-\ref{defn:epsilon-partition} are used in the context of chaining methode, e.g., Theorem~\ref{thm:bound-chaining-LD}, Corollary~\ref{cor:bound-hierarchical-chain} and Theorem~\ref{thm:chaining-single-loss-bound}.

The following is a technique assumption.
\begin{defn}[Separable Process]
\label{defn:separable-process}
The random process $\{X_t\}_{t\in\mathcal{T}}$ is called separable if there is a countable set $T_0\subseteq T$ s.t. $X_t\in\lim_{s\to t, s\in T_0} X_s$ for $\forall t \in T$ a.s.,  where $x\in\lim_{s\to t, s\in T_0} x_s$ means that there is a sequence $(s_n)$ in $T_0$ s.t. $s_n\to t$ and $x_{s_n}\to x$.
\end{defn}

\begin{defn}[Sub-Gaussian Process]
\label{defn:subgaussian-process}
The random process $\{X_t\}_{t\in T}$ on the metric space $(T,d)$ is called subgaussian if $\ex{}{X_t}=0$ for all $t\in T$ and $\ex{}{e^{\lambda(X_t-X_s)}}\leq e^{\frac{1}{2}\lambda^2d^2(t,s)}$ for all $t,s\in T$, $\lambda\geq 0$.
\end{defn}

\begin{defn}[Stochastic Chain \cite{zhou2022stochastic}]
\label{defn:stochastic-chain}
Let $(X_\mathcal{T},T)$ be a random process and random variable pair, where $T$ is a random variable in the index set $\mathcal{T}$. A sequence of random variables $\{T_k\}_{k=k_0}^\infty$ (with each distributed in $\mathcal{T}$) is called a stochastic chain of the pair $(X_\mathcal{T},T)$, if 1) $\lim_{k\to\infty}\ex{}{X_{T_k}}=\ex{}{X_{T}}$, 2) $\ex{}{X_{T_{k_0}}}=0$ and 3) $\{X_t\}_{t\in\mathcal{T}}-T-T_k-T_{k-1}$ is a Markov chain for every $k>k_0$.
\end{defn}

\begin{defn}[Increasing Sequence of $\epsilon$-Partition]
\label{defn:epsilon-partition}
A partition $\mathcal{P}=\{A_1, A_2, \dots, A_m\}$ of the set $T$ is called an $\epsilon$-partition of the metric space $(T,d)$ if for all $i=1,2,\dots,m$, $A_i$ can be contained within a ball of radius $\epsilon$. A sequence of partitions $\{\mathcal{P}_k\}_{k=m}^\infty$ of a set $T$ is called an increasing sequence if for any $k\geq m$ and each $A\in\mathcal{P}_{k+1}$, there exists $B\in\mathcal{P}_k$ s.t. $A\subseteq B$.
\end{defn}

The following lemmas are foundations of the most proofs in this paper.
\begin{lem}[Donsker-Varadhan (DV) variational representation of KL divergence {\citep[Theorem~3.5]{polyanskiy2019lecture}}]
    \label{lem:DV representation}
    Let $Q$, $P$ be probability measures on $\Theta$, for any bounded measurable function $f:\Theta\rightarrow \mathbb{R}$, we have
\[
\mathrm{D_{KL}}(Q||P) = \sup_{f} \ex{\theta\sim Q}{f(\theta)}-\ln\ex{\theta\sim P}{\exp{f(\theta)}}.
\]
\end{lem}

\begin{lem}[Hoeffding's Lemma {\citep{hoeffding1963probability}}]
    \label{lem:hoeffding}
     Let $X \in [a, b]$ be a bounded random variable with mean $\mu$. Then, for all $t\in\mathbb{R}$, we have $\ex{}{e^{tX}}\leq e^{t\mu+\frac{t^2(b-a)^2}{8}}$.
\end{lem}

\begin{lem}[Kantorovich-Rubinstein (KR) duality of Wasserstein distance {\citep{Villani2008}}]
    \label{lem:kr-duality}
    \label{lem:KR duality}
For any two distributions $P$ and $Q$, we have
\[
\mathbb{W}(P,Q)=\sup_{f\in\mathrm{1-Lip(\rho)}}\int_\mathcal{X} f dP - \int_\mathcal{X} f dQ,
\]
where the supremum is taken over all $1$-Lipschitz functions in the metric $d$, i.e. $|f(x)-f(x')|\leq d(x,x')$ for any $x,x'\in\mathcal{X}$.
\end{lem}

The following result is known in the previous work \cite{xu2017information}.
\begin{lem}[{\citet[Lemma~1]{xu2017information}}]
    \label{lem:subgaussian-dv}
    If $g(X', Y')$ is $\sigma$-subgaussian\footnote{A random variable $X$ is $\sigma$-subgaussian if for any $t$, $\ln {\mathbb E} \exp\left( t \left(
X- {\mathbb E}X
\right) \right) \le t^2\sigma^2/2$.} under $P_{X',Y'} = P_XP_Y$, then
    \[
    \abs{\ex{X,Y}{g(X, Y)}-\ex{X',Y'}{g(X', Y')}}\leq\sqrt{2\sigma^2I(X;Y)}.
    \]
\end{lem}

\section{Omitted Proofs and Additional Results in Section~\ref{sec:loss-difference}}
\subsection{Proof of Theorem~\ref{thm:bound-LD-cimi}}
The following proof shows that the proof of e-CMI bound in \citet{hellstrom2022a} can be adapted to the loss-difference MI bound, where we just replace the distribution $P_{L^+_i,L^-_i}$ by $P_{\Delta L_i}$.
\begin{proof}
    Let $U_i'$ be the independent copy of $U_i$ s.t. $U_i'\sim \mathrm{Bern}(1/2)$ and $U'\indp \Delta L_i | \widetilde{Z}$. Recall Lemma~\ref{lem:DV representation},
    \begin{align*}
        I(\Delta L_i;U_i|\widetilde{Z}=\tilde{z})=&\mathrm{D_{KL}}\pr{P_{\Delta L_i,U_i|\widetilde{Z}=\tilde{z}}||P_{\Delta L_i|\widetilde{Z}=\tilde{z}}P_{U'_i}}\\
        \geq&\sup_{t\in\mathbb{R}} \ex{\Delta L_i,U_i|\tilde{z}}{tg(\Delta L_i,U_i,\tilde{z})} - \ln\ex{\Delta L_i,U'_i|\tilde{z}}{e^{tg(\Delta L_i,U'_i,\tilde{z})}}.
    \end{align*}

    Recall that $w=\mathcal{A}(\tilde{z}_u)$, we now let $g(\Delta l_i,u_i,\tilde{z})=(-1)^{u_i}\Delta l_i=(-1)^{u_i}\pr{\ell(\mathcal{A}(\tilde{z}_u),\tilde{z}^-_{i})-\ell(\mathcal{A}(\tilde{z}_u),\tilde{z}^+_{i})}$.

    Then,
    \begin{align}
        I(\Delta L_i;U_i|\widetilde{Z}=\tilde{z})\geq&\sup_{t} \ex{\Delta L_i,U_i|\tilde{z}}{t(-1)^{U_i}\Delta L_i} - \ln\ex{\Delta L_i,U'_i|\tilde{z}}{e^{t(-1)^{U'_i}\Delta L_i}}\notag\\
        =&\sup_{t} \ex{\Delta L_i,U_i|\tilde{z}}{t(-1)^{U_i}\Delta L_i} - \ln\ex{\Delta L_i|\tilde{z}}{\ex{U'}{e^{t(-1)^{U'_i}\Delta L_i}\big\vert \Delta L_i=\Delta l_i}},\label{eq:independence-1}
    \end{align}
    where the last equality is by the conditional independence.

Since $\ex{U'}{t(-1)^{U'_i}\Delta l_i}=0$ and $(-1)^{U'_i}$ is bounded between $[-1,1]$, by Lemma~\ref{lem:hoeffding}, we have
\begin{align*}
    \ex{U_i'}{e^{t(-1)^{U'_i}\Delta l_i}}\leq e^{\frac{t^2\Delta l_i^2}{2}}
    \leq e^{\frac{t^2}{2}},
\end{align*}
where the second inequality is by the boundedness condition of the loss function, i.e. $\Delta l_i\in [-1,1]$.

Plugging the inequality above into Eq.~(\ref{eq:independence-1}), we have
\begin{align*}
     I(\Delta L_i;U_i|\widetilde{Z}=\tilde{z})\geq
     \sup_{t} \ex{\Delta L_i,U_i|\tilde{z}}{t(-1)^{U_i}\Delta L_i} - \frac{t^2}{2}.
\end{align*}

Then consider the case of $t>0$ and $t<0$ ($t=0$ is trivial), by AM–GM inequality (i.e. the arithmetic mean is greater than or equal to the geometric mean), the following is straightforward,
\[
\abs{\ex{\Delta L_i,U_i|\tilde{z}}{(-1)^{U_i}\Delta L_i}}\leq \sqrt{2 I(\Delta L_i;U_i|\widetilde{Z}=\widetilde{z})}.
\]

Notice that 
\begin{align}
    \abs{\mathrm{Err}} = \abs{\ex{S,W}{L_\mu(W)-L_S(W)}}=&\abs{\ex{\widetilde{Z},U,W}{L_{\widetilde{Z}\setminus\widetilde{Z}_U}(W)-L_{\widetilde{Z}_U}(W)}}\label{eq:start-disint-uncondi}\\
    \leq&\mathbb{E}_{\widetilde{Z}}\abs{\ex{U,W|\widetilde{Z}}{L_{\widetilde{Z}_{\overline{U}}}(W)-L_{\widetilde{Z}_U}(W)}}\label{ineq:for-disint}\\
    \leq&\frac{1}{n}\sum_{i=1}^n\mathbb{E}_{\widetilde{Z}}\abs{\ex{U_i,W|\widetilde{Z}}{(-1)^{U_i}\left(\ell(W,\widetilde{Z}^-_{i})-\ell(W,\widetilde{Z}^+_{i})\right)}}\notag\\
    =&\frac{1}{n}\sum_{i=1}^n\mathbb{E}_{\widetilde{Z}}\abs{\ex{U_i,\Delta L_i|\widetilde{Z}}{(-1)^{U_i}\Delta L_i}},\notag
\end{align}
wherein the two inequalities are by applying the Jensen's inequality to the absolute function.

Hence, putting everything together we have
\[
\abs{\mathrm{Err}}\leq \frac{1}{n}\sum_{i=1}^n\mathbb{E}_{\widetilde{Z}}\sqrt{2 I^{\widetilde{Z}}(\Delta L_i;U_i)}\leq\frac{1}{n}\sum_{i=1}^n\sqrt{2 I(\Delta L_i;U_i|\widetilde{Z})},
\]
where the second inequality is by applying the Jensen's inequality to the square root function. 

This completes the proof.
\end{proof}


\subsection{Proof of Theorem~\ref{thm:bound-LD-ucimi}}
By revisiting the proof of Theorem~\ref{thm:bound-LD-cimi}, particularly  Eq.~(\ref{eq:start-disint-uncondi}-\ref{ineq:for-disint}), we notice that if we do not move the expectation over $\widetilde{Z}$ outside of the absolute function, we will have a chance to get ride of the expectation over $\widetilde{Z}$ if we take the expectation over $\Delta L_i$.
\begin{proof}
    By the definition of the expected generalization error, we have
\begin{align}
    \abs{\mathrm{Err}} = \abs{\ex{S,W}{L_\mu(W)-L_S(W)}}
    =&\abs{\ex{\widetilde{Z},U,W}{L_{\widetilde{Z}\setminus\widetilde{Z}_U}(W)-L_{\widetilde{Z}_U}(W)}}\notag\\
    =&\abs{\frac{1}{n}\sum_{i=1}^n\ex{\widetilde{Z},U_i,W}{(-1)^{U_i}\pr{\ell(W,\widetilde{Z}^-_{i})-\ell(W,\widetilde{Z}^+_{i})}}}\label{ineq:for-ecmi}\\
    \leq&\frac{1}{n}\sum_{i=1}^n\abs{\ex{\Delta L_i,U_i}{(-1)^{U_i}\Delta L_i}}\label{ineq:for-LD-cmi}.
\end{align}

We know that $(-1)^{U'_i}\Delta L_i$ is bounded between $[-1,1]$, so it is a $1$-subgaussian random variable. Then, recall Lemma~\ref{lem:subgaussian-dv} and let $g(X,Y)=(-1)^{U_i}\Delta L_i$, we have
\[
\abs{\ex{\Delta L_i,U_i}{(-1)^{U_i}\Delta L_i}-\ex{\Delta L_i,U'_i}{(-1)^{U'_i}\Delta L_i}}\leq \sqrt{2 I(\Delta L_i;U_i)}.
\]

Since $\ex{\Delta L_i,U'_i}{(-1)^{U'_i}\Delta L_i}=0$, plugging the inequality above into Eq.~(\ref{ineq:for-LD-cmi}), we have
\[
\abs{\mathrm{Err}}\leq\frac{1}{n}\sum_{i=1}^n\abs{\ex{\Delta L_i,U_i}{(-1)^{U_i}\Delta L_i}}\leq \frac{1}{n}\sum_{i=1}^n\sqrt{2 I(\Delta L_i;U_i)}.
\]







This concludes the proof.
\end{proof}

\subsection{Proof of Theorem~\ref{thm:bound-LD-wass}}

\begin{proof}
    Recall Eq.~(\ref{ineq:for-LD-cmi}), we could also obtain
\begin{align}
    \abs{\mathrm{Err}} \leq&\frac{1}{n}\sum_{i=1}^n\abs{\ex{\Delta L_i,U_i}{(-1)^{U_i}\Delta L_i}}\notag=\frac{1}{n}\sum_{i=1}^n\abs{\ex{\Delta L_i,U_i}{(-1)^{U_i}\Delta L_i}-\ex{\Delta L'_i,U_i}{(-1)^{U_i}\Delta L'_i}},
\end{align}
where $\Delta L'_i$ is an independent copy of $\Delta L_i$ (i.e. $\Delta L'_i\sim P_{\Delta L_i}$ and $\Delta L'_i\indp U_i$) and the second equality holds since $\ex{\Delta L'_i,U_i}{(-1)^{U_i}\Delta L'_i}=0$.

Then, by Jensen's inequality, we move the expectation over $U_i$ and the average outside the absolute function,
\begin{align*}
    \abs{\mathrm{Err}} \leq&\frac{1}{n}\sum_{i=1}^n\ex{U_i}{\abs{\ex{\Delta L_i|U_i}{(-1)^{U_i}\Delta L_i}-\ex{\Delta L'_i}{(-1)^{U_i}\Delta L'_i}}\Big\vert U_i=u_i}.
\end{align*}

Notice that for any fixed $u_i$, we have
\begin{align*}
    \ex{\Delta L_i|U_i=u_i}{(-1)^{u_i}\Delta L_i}=&\int_{-1}^1(-1)^{u_i}\Delta \ell_i dP_{\Delta L_i|U_i=u_i}(\Delta \ell_i),\\
    \ex{\Delta L'_i}{(-1)^{u_i}\Delta L'_i}=&\int_{-1}^1(-1)^{u_i}\Delta \ell'_i dP_{\Delta L_i}(\Delta \ell'_i).
\end{align*}

Also, noting that $f(x)=x$ is a 1-Lipschitz function, i.e. $\abs{(-1)^{u_i}\Delta L_i-(-1)^{u_i}\Delta L'_i}\leq \abs{\Delta L_i-\Delta L'_i}$ holds trivially.

Recall the KR duality of Wasserstein distance (i.e. Lemma~\ref{lem:kr-duality}), we have
\begin{align*}
    \abs{\mathrm{Err}} \leq&\frac{1}{n}\sum_{i=1}^n\ex{U_i}{\abs{\ex{\Delta L_i|U_i}{(-1)^{U_i}\Delta L_i}-\ex{\Delta L'_i}{(-1)^{U_i}\Delta L'_i}}\Big\vert U_i=u_i}
    \leq \frac{1}{n}\sum_{i=1}^n\ex{U_i}{\mathbb{W}(P_{\Delta L_i|U_i},P_{\Delta L_i})}.
\end{align*}
This concludes the proof.
\end{proof}

\subsection{Proof of Theorem~\ref{thm:channel-interpolating}}
\begin{proof}
    The channel capacity can be obtained from Lemma~\ref{lem:channel-capacity} by letting $\epsilon_i=0$ and changing the unit of bit to the unit of nat (i.e. replacing logarithm base of $2$ to the base of $e$ by $\ln{x}=\ln{2}\log_2{x}$).


    Furthermore, the value of $1-\alpha_i$ reflects the chance that the interpolating learning algorithm $\mathcal{A}$ will make an error (i.e. $\ell(W,Z'_i)=1$) for a testing instance, or equivalently,
    \begin{align*}
        \ex{W,Z'_i}{\mathbbm{1}_{f_W(X_i')\neq Y'}}=\ex{U_i,L_i}{L_{i,\overline{U}_i}}
        =&\frac{\ex{L^-_i|U_i=0}{L^-_i}+\ex{L^+_i|U_i=1}{L^+_i}}{2}\\
        =&\frac{P(\Delta L_i=1|U_i=0)+P(\Delta L_i=-1|U_i=1)}{2}
        =1-\alpha_i.
    \end{align*}
 Thus, combining the equality above with  $C=I(U_i;\Delta L_i)=(1-\alpha_i)\cdot\ln{2}$, we have
    \[
    \abs{\mathrm{Err}}=\ex{W}{L_\mu(W)}=\frac{1}{n}\sum_{i=1}^n\ex{W,Z_i'}{\mathbbm{1}_{f_W(X_i')\neq Y_i'}}=\frac{1}{n}\sum_{i=1}^n (1-\alpha_i) = \frac{1}{n\ln{2}}\sum_{i=1}^n I(U_i;\Delta L_i).
    \]
    This completes the proof.
\end{proof}

\subsection{Proof of Theorem~\ref{thm:bound-chaining-LD}}
\label{sec:proof-chaining-LD}
\begin{proof}
    Let $E^i_{\Delta L_i}=\mathrm{Err}^i(\Delta L_i)=(-1)^{U_i}\Delta L_i$, then for any integers $k_1$ and $k_0$ such that $k_1>k_0$, we have
    \[
    E^i_{\Delta L_i}=E^i_{\Delta L_{i,k_0}}+\sum_{k=k_0+1}^{k_1}(E^i_{\Delta L_{i,k}}-E^i_{\Delta L_{i,k-1}}) + E^i_{\Delta L_i}-E^i_{\Delta L_{i,k_1}}.
    \]

    By the definition of the stochastic chain (i.e. Definition~\ref{defn:stochastic-chain}), we know that $\ex{E^i_{\Delta L_{i,k_0}}}{E^i_{\Delta L_{i,k_0}}}=0$ and $\lim_{k_1\to \infty}E^i_{\Delta L_{i,k_1}}=E^i_{\Delta L_i}$. Therefore, let $k_1\to\infty$ and taking expectation over $(U_i,\Delta L_i)\sim P_{U_i,\Delta L_i}$ for both sides of the equation above, we have
    \begin{align}
        \ex{U_i,\Delta L_i}{E^i_{\Delta L_i}}=\sum_{k=k_0+1}^{\infty}\ex{U_i,\Delta L_{i,k},\Delta L_{i,k-1}}{E^i_{\Delta L_{i,k}}-E^i_{\Delta L_{i,k-1}}}.
        \label{eq:chain-decomposition}
    \end{align}

    Let $U_i'$ be an independent copy of $U_i$ and recall Lemma~\ref{lem:DV representation}, we have
    \begin{align*}
        &\ex{\Delta L_{i,k},\Delta L_{i,k-1}}{\kl{P_{U_i|\Delta L_{i,k},\Delta L_{i,k-1}}}{P_{U_i'}}}\\
        \geq& \ex{\Delta L_{i,k},\Delta L_{i,k-1}}{\sup_{t>0} t\ex{U_i|\Delta L_{i,k},\Delta L_{i,k-1}}{E^i_{\Delta L_{i,k}}-E^i_{\Delta L_{i,k-1}}}-\ln \ex{U'_i}{e^{t\pr{E^i_{\Delta L_{i,k}}-E^i_{\Delta L_{i,k-1}}}}}}\\
        \geq& \sup_{t>0} t\ex{U_i,\Delta L_{i,k},\Delta L_{i,k-1}}{E^i_{\Delta L_{i,k}}-E^i_{\Delta L_{i,k-1}}}-\mathbb{E}_{\Delta L_{i,k},\Delta L_{i,k-1}}\ln \ex{U'_i}{e^{t(-1)^{U'_i}\pr{\Delta L_{i,k}-\Delta L_{i,k-1}}}},
    \end{align*}
    where the second inequality is by applying Jensen's inequality to the supremum.

    Notice that the LHS above is equivalent to $I(\Delta L_{i,k},\Delta L_{i,k-1};U_i)$. Since $(-1)^{U'_i}$ is bounded between $[-1,1]$, by Lemma~\ref{lem:hoeffding}, we have
     \begin{align*}
        I(\Delta L_{i,k},\Delta L_{i,k-1};U_i) 
        \geq& \sup_{t>0} t\ex{U_i,\Delta L_{i,k},\Delta L_{i,k-1}}{E^i_{\Delta L_{i,k}}-E^i_{\Delta L_{i,k-1}}}-\mathbb{E}_{\Delta L_{i,k},\Delta L_{i,k-1}}\ln e^{\frac{t^2(\Delta L_{i,k}-\Delta L_{i,k-1})^2}{2}}.
    \end{align*}

    Thus, let $d(\Delta L_{i,k},\Delta L_{i,k-1})=\abs{\Delta L_{i,k}-\Delta L_{i,k-1}}$, we have
    \[
    \ex{U_i,\Delta L_{i,k},\Delta L_{i,k-1}}{E^i_{\Delta L_{i,k}}-E^i_{\Delta L_{i,k-1}}}\leq\sqrt{2\ex{\Delta L_{i,k},\Delta L_{i,k-1}}{d^2(\Delta L_{i,k},\Delta L_{i,k-1})}I(\Delta L_{i,k},\Delta L_{i,k-1};U_i)}.
    \]

    Plugging the inequality above into Eq.~(\ref{eq:chain-decomposition}) and taking average over $i$, we have,
    \begin{align*}
        \mathrm{Err}=&\frac{1}{n}\sum_{i=1}^n \ex{U_i,\Delta L_i}{E^i_{\Delta L_i}}\\
        \leq&\frac{1}{n}\sum_{i=1}^n\sum_{k=k_0+1}^{\infty}\sqrt{2\ex{\Delta L_{i,k},\Delta L_{i,k-1}}{d^2(\Delta L_{i,k},\Delta L_{i,k-1})}I(\Delta L_{i,k},\Delta L_{i,k-1};U_i)}.
    \end{align*}

    From the third point of Definition~\ref{defn:stochastic-chain}, we know that $U_i-\Delta L_{i}-\Delta L_{i,k}-\Delta L_{i,k-1}$ is a Markov chain, so $I(\Delta L_{i,k},\Delta L_{i,k-1};U_i)=I(\Delta L_{i,k};U_i)+I(\Delta L_{i,k-1};U_i|\Delta L_{i,k})=I(\Delta L_{i,k};U_i)$. This gives us the final form of the bound,
    \[
    \mathrm{Err}\leq\frac{1}{n}\sum_{i=1}^n\sum_{k=k_0+1}^{\infty}\sqrt{2\ex{\Delta L_{i,k},\Delta L_{i,k-1}}{d^2(\Delta L_{i,k},\Delta L_{i,k-1})}I(\Delta L_{i,k};U_i)}.
    \]
    This concludes the proof.
\end{proof}

\subsection{Additional Result: Chained MI Bound for Bounded Loss}
\begin{cor}
    \label{cor:bound-hierarchical-chain}
    Let $2^{-k_0}\geq\mathrm{diam}(\Gamma)$ and let $\{\mathcal{P}_k\}_{k=k_0}^\infty$ be an increasing sequence of partitions of 
  $\Gamma$, where for each $k\geq k_0$, $\mathcal{P}_k$ is a $2^{-k}$-partition of $(\Gamma,d)$. Let $\Delta L_{i,k}$ be the center of the covering ball of the partition cell that $\Delta L_{i}$ belongs to the partition $\mathcal{P}_k$, then
    \[
    \mathrm{Err}\leq \frac{3}{n}\sum_{i=1}^n\sum_{k=k_0}^\infty2^{-k}\sqrt{2I(\Delta L_{i,k};U_i)}.
    \]
\end{cor}
\begin{proof}
    By the triangle inequality, we have
    $
        d(\Delta L_{i,k},\Delta L_{i,k-1})
        \leq d(\Delta L_{i,k},\Delta L_{i})+d(\Delta L_{i},\Delta L_{i,k-1}).
    $
    
    Since $\mathcal{P}_k$ is a $2^{-k}$ partition, $d(\Delta L_{i,k},\Delta L_i)\leq 2^{-k}$, then
    $
    d(\Delta L_{i,k},\Delta L_{i})+d(\Delta L_{i},\Delta L_{i,k-1})\leq 2^{-k}+2^{-(k-1)}=3\times 2^{-k}.
    $
    Plugging this into Theorem~\ref{thm:bound-chaining-LD}, we have
    \[
    \mathrm{Err}\leq\frac{1}{n}\sum_{i=1}^n\sum_{k=k_0+1}^{\infty}\sqrt{2\ex{\Delta L_{i,k},\Delta L_{i,k-1}}{d^2(\Delta L_{i,k},\Delta L_{i,k-1})}I(\Delta L_{i,k};U_i)}\leq\frac{1}{n}\sum_{i=1}^n\sum_{k=k_0+1}^{\infty}3\times 2^{-k}\sqrt{2I(\Delta L_{i,k};U_i)}.
    \]
    This completes the proof.
\end{proof}

\section{Omitted Proofs and Additional Results in Section~\ref{sec:single-loss}}
\subsection{Proof of Lemma~\ref{lem:standard-symmetric}}
\begin{proof}
    By the definition of $\mathrm{Err}$, we can decompose it into two terms,
    \begin{align*}
        \mathrm{Err}=\frac{1}{n}\sum_{i=1}^n\ex{L^+_{i},L^-_{i},U_i}{L_{i,\overline{U}_i}-L_{i,{U}_i}}
        =&\frac{1}{n}\sum_{i=1}^n\ex{L^+_{i},L^-_{i},U_i}{(-1)^{U_i}(L^-_{i}-L^+_{i})}\\
        =&\frac{1}{n}\sum_{i=1}^n\br{\ex{L^-_{i},U_i}{(-1)^{U_i} L^-_{i}}+\ex{L^+_{i},U_i}{-(-1)^{U_i} L^+_{i}}}\\
        =&\frac{1}{n}\sum_{i=1}^n\br{\ex{L^-_{i},U_i}{(-1)^{U_i} L^-_{i}}+\ex{L^+_{i},\overline{U}_i}{(-1)^{\overline{U}_i} L^+_{i}}},
    \end{align*}
    where the last equality is by $-(-1)^{U_i} L_i^+=(-1)^{\overline{U}_i} L_i^+$.
    We now show that the following holds
    \[
    \ex{L^-_{i},U_i}{(-1)^{U_i} L^-_{i}}=\ex{L^+_{i},\overline{U}_i}{(-1)^{\overline{U}_i} L^+_{i}}.
    \]
    
    Recall that $\widetilde{Z}$ and $U$ are i.i.d drawn from $\mu^{2n}$ and the Bernoulli distribution, respectively, and $\widetilde{Z} \indp U$. Usually, a learning algorithm may depend on the order of training instances (i.e. for $i \neq j$, even if two training instances satisfy $z_i=z_j$, $P_{W|z_i}$ may not be the same with $P_{W|z_j}$), but it should be invariant to the \emph{row index} of the supersample $\widetilde{Z}$, then the distribution $P_{L^-_{i},U_i}$ and the distribution $P_{L^+_{i},U_i}$ have some symmetric property, namely, $P_{L^-_{i}|U_i=1}=P_{L^+_{i}|U_i=0}$ and $P_{L^-_{i}|U_i=0}=P_{L^+_{i}|U_i=1}$. Or equivalently, the distribution of the training loss (or testing loss) of the $i$th training instance (or testing instance) is invariant to $U_i$.  Hence, we have $P_{L^-_{i}}=P_{L^+_{i}}$, we say $L^-_{i}$ and $L^+_{i}$ are identically but not independently distributed. 
    Then,
    \begin{align*}
        &\ex{L^-_{i}|U_i=0}{ L^-_{i}}=\int_{0}^1 \ell^-_{i} dP_{L^-_{i}|U_i=0}(\ell^-_{i})=\int_{0}^1 \ell^+_{i} dP_{L^+_{i}|U_i=1}(\ell^+_{i})=\ex{L^+_{i}|U_i=1}{ L^+_{i}},\\
        &\ex{L^-_{i}|U_i=1}{ -L^-_{i}}=\int_{0}^1 -\ell^-_{i} dP_{L^-_{i}|U_i=1}(\ell^-_{i})=\int_{0}^1 -\ell^+_{i} dP_{L^+_{i}|U_i=0}(\ell^+_{i})=\ex{L^+_{i}|U_i=0}{- L^+_{i}}.
    \end{align*}
    These give us
    \begin{align*}
        \ex{L^-_{i},U_i}{(-1)^{U_i} L^-_{i}}=\frac{\ex{L^-_{i}|U_i=0}{ L^-_{i}}+\ex{L^-_{i}|U_i=1}{ -L^-_{i}}}{2}=&\frac{\ex{L^+_{i}|U_i=1}{ L^+_{i}}+\ex{L^+_{i}|U_i=0}{ -L^+_{i}}}{2}\\
        =&\ex{L^+_{i},\overline{U}_i}{(-1)^{\overline{U}_i} L^+_{i}}.
    \end{align*}
 Therefore, 
\begin{align*}
    \mathrm{Err}=\frac{1}{n}\sum_{i=1}^n\br{\ex{L^-_{i},U_i}{(-1)^{U_i} L^-_{i}}+\ex{L^+_{i},\overline{U}_i}{(-1)^{\overline{U}_i} L^+_{i}}}=&\frac{2}{n}\sum_{i=1}^n\ex{L^-_{i},U_i}{(-1)^{U_i}L^-_{i}}
    =&\frac{2}{n}\sum_{i=1}^n\ex{L^+_{i},\overline{U}_i}{(-1)^{\overline{U}_i}L^+_{i}}.
\end{align*}
Notice that both $(-1)^{U_i}$ and $(-1)^{\overline{U}_i}$ are Rademacher variables. This conclude the proof.
\end{proof}

\subsection{Proof of Theorem~\ref{bound-single-loss-sq}}
\begin{proof}
    Notice that $2(-1)^{U'_i}L_i^+$ is bounded between $[-2,2]$, it is a subgaussian random variable with the variance proxy $\sigma=2$. Let the measurable function $g(L_i^+,U_i)$ in Lemma~\ref{lem:subgaussian-dv} be $2(-1)^{U_i}L^+_{i}$, then $g(L_i^+,U_i)$ is $2$-subgaussian under $P_{U_i}P_{L_i^+}$, we have
    \begin{align*}
        \abs{2\ex{L^+_{i},U_i}{(-1)^{U_i}L^+_{i}}-2\ex{L^+_{i},U'_i}{(-1)^{U'_i}L^+_{i}}}\leq 2\sqrt{2I(L_i^+;U_i)}.
    \end{align*}

    Since $\ex{L^+_{i},U'_i}{(-1)^{U'_i}L^+_{i}}=\frac{\ex{L^+_i}{L^+_{i}}-\ex{L^+_i}{L^+_{i}}}{2}=0$, then
    \begin{align*}
        \abs{2\ex{L^+_{i},U_i}{(-1)^{U_i}L^+_{i}}}\leq 2\sqrt{2I(L_i^+;U_i)}.
    \end{align*}

Recall Lemma~\ref{lem:standard-symmetric},
    \[
    \mathrm{Err}=\frac{2}{n}\sum_{i=1}^n\ex{L^+_{i},\overline{U}_i}{(-1)^{\overline{U}_i}L^+_{i}}.
    \]

    Notice that $\overline{U}_i$ and $U_i$ are one-to-one mapping, so using any of them will give the same mutual information. Thus, by applying Jensen's inequality to the absolute function, we have
    \[
    \abs{\mathrm{Err}}\leq\frac{1}{n}\sum_{i=1}^n\abs{2\ex{L^+_{i},U_i}{(-1)^{U_i}L^+_{i}}}\leq \frac{2}{n}\sum_{i=1}^n\sqrt{2I(L_i^+;U_i)}.
    \]

    In addition, 
    to obtain the second inequality in the theorem, we first invoke the independence between $\widetilde{Z}$ and $U_i$, $I(L_i^+;U_i)+I(\widetilde{Z};U_i|L_i^+)=I(L_i^+,\widetilde{Z};U_i)=I(L_i^+;U_i|\widetilde{Z})$ and then use the DPI, $I(L_i^+;U_i|\widetilde{Z})\leq I(f_W(X);U_i|\widetilde{Z})$, we have
    \[
    \left|\mathrm{Err}\right| \leq\frac{2}{n}\sum_{i=1}^n\sqrt{2I(L^+_{i};U_i)}\leq \frac{2}{n}\sum_{i=1}^n\sqrt{2I(f_W(X^+_{i});U_i|\widetilde{Z})}.
    \]
    This completes the proof.
\end{proof}

\subsection{Proof of Theorem~\ref{thm:channel-capacity-z}}
\begin{proof}
Notice that 
\begin{align*}
        \ex{W,Z'_i}{\mathbbm{1}_{f_W(X_i')\neq Y'}}=\ex{U_i,L_i}{L_{i,\overline{U}_i}}
        =\frac{\ex{L^-_i|U_i=0}{L^-_i}+\ex{L^+_i|U_i=1}{L^+_i}}{2}=\ex{L^+_i|U_i=1}{L^+_i}=P(L^+_i=1|U_i=1)
        =1-q_i.
    \end{align*}
    Hence, $L_\mu=\sum_{i=1}^n\frac{1-q_i}{n}$.

    For each $i$, we have
    $
    I(L^+_{i};U_i)=H(L^+_{i})-H(L^+_{i}|U_i)=H(\frac{1-q_i}{2})-\frac{1}{2}H(1-q_i)\leq H(\frac{1-q_i}{2}).
    $ Since the entropy function $H(\cdot)$ is a concave function, we have
    $
    \frac{1}{n}\sum_{i=1}^nI(L^+_{i};U_i)\leq \frac{1}{n}\sum_{i=1}^nH(\frac{1-q_i}{2})\leq H(\frac{L_\mu}{2})
    $.
\end{proof}

\subsection{Additional Result: Chained MI Bound Based on Single-Loss}
\label{sec:chianed-single-loss}
When the loss is not discrete or even not bounded, let $\xi^i(\ell^+_i)\triangleq\varepsilon_i \ell^+_i$ be a random process and let $\mathcal{L}$ be the domain of $L_i^+$. Similar to Theorem~\ref{thm:bound-chaining-LD}, we can also have the corresponding chained bound of Theorem~\ref{bound-single-loss-sq}.
\begin{thm}
    \label{thm:chaining-single-loss-bound}
    For each $i\in [n]$, we assume $\{L^+_{i,k}\}_{k=k_0}^\infty$ is a stochastic chain of $(\xi^i(\ell^+_i)\}_{\ell^+_i\in\mathcal{L}},L^+_i)$, then
    \[
    \mathrm{Err}\leq \frac{2}{n}\sum_{i=1}^n\sum_{k=k_0}^\infty\sqrt{2\ex{}{d^2( L^+_{i,k}, L^+_{i,k-1})}I( L^+_{i,k};U_i)},
    \]
    where the RHS expectation is taken over $(L^+_{i,k}, L^+_{i,k-1})$.
\end{thm}
This theorem can be obtained by following the same development with the proof in Section~\ref{sec:proof-chaining-LD}.

\subsection{Proof of Lemma~\ref{lem:weighted-symmetric}}

\begin{proof}
A key step is the second equality of the following
\begin{align*}
    {\mathrm{Err}}_{C_1}=&\frac{1}{n}\sum_{i=1}^n\ex{L^-_{i},L^+_{i},U_i}{L_{i,\overline{U}_i}-(1+C_1)L_{i,{U}_i}}\\
    =&\frac{1}{n}\sum_{i=1}^n\ex{L^-_{i},L^+_{i},U_i}{\pr{1+\frac{C_1}{2}}(L_{i,\overline{U}_i}-L_{i,{U}_i})-\frac{C_1}{2}L_{i,\overline{U}_i}-\frac{C_1}{2}L_{i,{U}_i}}\\
    =&\frac{2+C_1}{2n}\sum_{i=1}^n\br{\ex{L^-_{i},U_i}{(-1)^{U_i} L^-_{i}-\frac{C_1}{C_1+2}L^-_{i}}+\ex{L^+_{i},U_i}{-(-1)^{U_i} L^+_{i}-\frac{C_1}{C_1+2}L^-_{i}}}.
\end{align*}

Recall that $P_{L^-_{i}}=P_{L^+_{i}}$, we have $\ex{L^-_{i}}{\frac{C_1}{C_1+2}L^-_{i}}=\ex{L^+_{i}}{\frac{C_1}{C_1+2}L^+_{i}}$.
Also, noting that $\ex{L^-_{i}|U_i=0}{ L^-_{i}}=\ex{L^+_{i}|U_i=1}{ L^+_{i}}$ and $\ex{L^-_{i}|U_i=1}{ -L^-_{i}}=\ex{L^+_{i}|U_i=0}{ -L^+_{i}}$, we have
\begin{align}
    {\mathrm{Err}}_{C_1}=\frac{2+C_1}{n}\sum_{i=1}^n\ex{L^+_{i},\overline{U}_i}{(-1)^{\overline{U}_i} L^+_{i}-\frac{C_1}{C_1+2}L^+_{i}}=\frac{2+C_1}{n}\sum_{i=1}^n\ex{L^+_{i},\tilde{\varepsilon}_i}{ \tilde{\varepsilon}_iL^+_{i}},\label{eq:weighted-error}
\end{align}
where $\tilde{\varepsilon}_i=\varepsilon_i-\frac{C_1}{C_1+2}$ and $\varepsilon_i\sim \mathrm{Unif}\{-1,+1\}$ is the Rademacher variable. In this case, $\tilde{\varepsilon}_i$ is called the \emph{shifted} Rademacher variable and its mean is $-\frac{C_1}{C_1+2}$.
\end{proof}

\subsection{Proof of Theorem~\ref{thm:bound-fast-rate-rademacher}}
\begin{proof}
     Since $\tilde{\varepsilon}_i$ is obtained by a bijection function of $U_i$, they two can be replaced by each other in the mutual information. Recall Lemma~\ref{lem:DV representation} and let the measurable function $g$ be $t(C_1+2)\tilde{\varepsilon}_i L_i^+$,  we have
    \begin{align}
    I(L^+_{i};U_i)=I(L^+_{i};\tilde{\varepsilon}_i)=&\mathrm{D_{KL}}\pr{P_{L^+_{i},\tilde{\varepsilon}_i}||P_{L^+_{i}}P_{\tilde{\varepsilon}'_i}}\\
        \geq&\sup_{t>0} \ex{L^+_{i},\tilde{\varepsilon}_i}{t(C_1+2)\tilde{\varepsilon}_i L^+_{i}} - \ln\ex{ L^+_{i},\tilde{\varepsilon}'_i}{e^{t(C_1+2)\tilde{\varepsilon}'_i L^+_{i}}}.
        \label{ineq:DV-fast}
    \end{align}

We hope to have
   \begin{align}
       \ex{L^+_{i},\tilde{\varepsilon}'_i}{e^{t(C_1+2)\tilde{\varepsilon}'_i L^+_{i}}}\leq 1.
       \label{ineq:bound-log-moment-1}
   \end{align}

Since $\tilde{\varepsilon}'_i$ is independent of $L^+_{i}$, and $P(\tilde{\varepsilon}_i=\frac{2}{C_1+2})=P(\tilde{\varepsilon}_i=\frac{-2(C_1+1)}{C_1+2})=\frac{1}{2}$, then
    \[
    \ex{L^+_{i},\tilde{\varepsilon}'_i}{e^{t(C_1+2)\tilde{\varepsilon}'_i L^+_{i}}}=\frac{\ex{ L^+_{i}}{e^{-2t(C_1+1) L^+_{i}}+e^{2tL^+_{i}}}}{2}.
    \]

    Notice that $e^{-2t(C_1+1) L^+_{i}}+e^{2tL^+_{i}}$ is the summation of two convex function, which is still a convex function, so the maximum value of this function is achieved at the endpoints of the bounded domain. Recall that $L^+_{i}\in [0,1]$, we now consider two cases, i) when $L^+_{i}=0$, we have $e^{-2t(C_1+1) L^+_{i}}+e^{2tL^+_{i}}=2$;
    ii)when $L^+_{i}=1$, we need to require
    $
    e^{-2t(C_1+1)}+e^{2t}\leq 2
    $
    s.t. Eq~(\ref{ineq:bound-log-moment-1}) can hold. Note that this inequality implies that $t\leq \frac{\ln{2}}{2}$.

     Replacing $t$ by $C_2$, let the values of $C_1,C_2$ be taken from the domain of $\{C_1,C_2|C_1,C_2>0, e^{-2C_2(C_1+1)}+e^{2C_2}\leq 2\}$, so Eq.~(\ref{ineq:bound-log-moment-1}) will hold. Under this condition, by re-arranging the terms in Eq.~(\ref{ineq:DV-fast}), we have
    \[
    (C_1+2)\ex{L^+_{i},\tilde{\varepsilon}_i}{\tilde{\varepsilon}_i L^+_{i}}\leq \frac{I(L^+_{i};U_i)}{C_2}.
    \]
    
    Plugging the  inequality above into Eq.~(\ref{eq:weighted-error}), we have
    \[
    {\mathrm{Err}}_{C_1}=L_\mu-(1+C_1)L_n=\frac{2+C_1}{n}\sum_{i=1}^n\br{\ex{L^+_{i},\tilde{\varepsilon}_i}{ \tilde{\varepsilon}_iL^+_{i}}}\leq \sum_{i=1}^n \frac{I(L^+_{i};U_i)}{C_2n}.
    \]

    Thus, the following inequality can be obtained,
    \begin{align}
        L_\mu\leq \min_{C_1,C_2>0,e^{2C_2}+e^{-2C_2(C_1+1)}\leq 2}  (1+C_1)L_n+ \sum_{i=1}^n \frac{I(L^+_{i};U_i)}{C_2n}.
        \label{ineq:fast-rate-optimal-1}
    \end{align}
    





    We can also optimize the parameters $C_1,C_2$ by relaxing the condition of $e^{-2C_2(C_1+1)}+e^{2C_2}\leq 2$.
    By invoking $e^{x}\geq x+1$ and $e^{-x}\leq \frac{1}{1+x}$ for $x>-1$, and $e^x\leq \frac{1}{1-x}$ for $x<1$, it's sufficient to have
    \begin{align}
        \frac{1}{1+2C_2(C_1+1)}+\frac{1}{1-2C_2}\leq 2,\qquad \text{and $0< C_2 < \frac{1}{2}$.} \label{ineq:quadratic-function}
    \end{align}

    Solving Eq~(\ref{ineq:quadratic-function}) gives us $0<C_2\leq \frac{C_1}{4(C_1+1)}$. Since $\frac{C_1}{4(C_1+1)}\leq\frac{1}{4}<\frac{1}{2}$, then Eq~(\ref{ineq:quadratic-function}) holds when $C_2\in (0, \frac{C_1}{4(C_1+1)}]$. Notice that $\frac{C_1}{4(C_1+1)}$ is also smaller than $\frac{\ln{2}}{2}$.

    

    Therefore, we obtain 
    \begin{align*}
        L_\mu\leq &\min_{C_1,C_2>0,e^{2C_2}+e^{-2C_2(C_1+1)}\leq 2}  (1+C_1)L_n+ \sum_{i=1}^n \frac{I(L^+_{i};U_i)}{C_2n}\\
        \leq&\min_{C_1>0,0<C_2\leq \frac{C_1}{4(C_1+1)}}  (1+C_1)L_n+ \sum_{i=1}^n \frac{I(L^+_{i};U_i)}{C_2n}\\
        =&L_n+\sum_{i=1}^n \frac{4I(L^+_{i};U_i)}{n}+4\sqrt{\sum_{i=1}^n \frac{L_nI(L^+_{i};U_i)}{n}},
    \end{align*}
    where the last equality is achieved when $C_1=2\sqrt{\sum_{i=1}^n \frac{I(L^+_{i};U_i)}{nL_n}}$ and $C_2 = \frac{C_1}{4(C_1+1)}$.



    For the second part of the theorem,
    if $\mathcal{A}$ is an interpolating algorithm, then $L_n=0$, in which case we can let $C_1$ be arbitrarily large. 

    Recall that we hope 
    \[
    e^{-2C_2(C_1+1)}+e^{2C_2}\leq 2.
    \]

    This can be satisfied by letting $C_2=\frac{\ln{2}}{2}$ and $C_1\rightarrow \infty$.
    
    Thus, the interpolating single-loss MI bound is 
     \[
    L_\mu\leq \sum_{i=1}^n \frac{2I(L^+_{i};U_i)}{n\ln{2}}.
    \]
    This completes the proof.
\end{proof}

\subsection{Proof of Lemma~\ref{lem:empirical-variance}}
\begin{proof}
    By the definition of $\gamma$-variance, and notice that $L_S(W)=\frac{1}{n}\sum_{i=1}^n\ell(W,Z_i)$, we have
    \begin{align*}
        V(\gamma)=&\ex{W,S}{\frac{1}{n}\sum_{i=1}^n\pr{\ell(W,Z_i)-(1+\gamma)L_S(W)}^2}\\
        =&\ex{W,S}{\frac{1}{n}\sum_{i=1}^n\pr{\ell^2(W,Z_i)-2(1+\gamma)\ell(W,Z_i)L_S(W)+(1+\gamma)^2L^2_S(W)}}\\
        =&\ex{W,S}{\frac{1}{n}\sum_{i=1}^n\ell^2(W,Z_i)}-\ex{W,S}{(1-\gamma^2)L^2_S(W)}\\
        =&L_n-(1-\gamma^2)\ex{W,S}{L^2_S(W)}
    \end{align*}
    where the last equality is due to the fact that the loss is the zero-one loss (i.e. $\ell^2(\cdot,\cdot)=\ell(\cdot,\cdot)$).
\end{proof}

\subsection{Proof of Lemma~\ref{lem:vaiance-symmetric}}
\label{sec:symmetric-variance}
\begin{proof}
    By Lemma~\ref{lem:empirical-variance} and $\gamma\in (0,1)$, we notice that,
    \begin{align}
        \mathrm{Err}-C_1V(\gamma)=&L_\mu-L_n-C_1L_n+C_1(1-\gamma^2)\ex{W,S}{L^2_S(W)}\notag\\
        \leq& L_\mu-(1+C_1)L_n+C_1(1-\gamma^2)\ex{W,S}{L_S(W)}\label{ineq:variance-relax}\\
        =&L_\mu-(1+C_1\gamma^2)L_n,\label{ineq:variance-fast-rate}
    \end{align}
    where the inequality is because that $L_S(W)\in[0,1]$ (i.e. $L^2_S(W)\leq L_S(W)$).

    Noting that in Eq~(\ref{ineq:variance-fast-rate}), $\mathrm{Err}-C_1V(\gamma)$ is upper bounded by a weighted generalization error. Thus, we can then directly apply Lemma~\ref{lem:weighted-symmetric} by choosing $C_1\gamma^2$ as the trade-off coefficient (i.e. replacing $C_1$ in Lemma~\ref{lem:weighted-symmetric} by $C_1\gamma^2$), which gives us
    \[
    \mathrm{Err}-C_1V(\gamma)\leq\frac{2+C_1\gamma^2}{n}\sum_{i=1}^n\ex{L^+_{i},\tilde{\varepsilon}_i}{ \tilde{\varepsilon}_iL^+_{i}},
    \]
    where $\tilde{\varepsilon}_i=\varepsilon_i-\frac{C_1\gamma^2}{C_1\gamma^2+2}$.
\end{proof}

\subsection{Proof of Theorem~\ref{thm:bound-variance}}
\label{sec:proof-variance-bound}
\begin{proof}
    The RHS of Eq.~(\ref{ineq:variance-fast-rate}) in the proof of Lemma~\ref{lem:vaiance-symmetric} has already been bounded in Theorem~\ref{thm:bound-fast-rate-rademacher} by regarding $C_1\gamma^2$ as the weighted parameter $C_1$ in Theorem~\ref{thm:bound-fast-rate-rademacher}. Then, there exist $C_1,C_2>0$ s.t.
    \begin{align*}
        \mathrm{Err}-C_1V(\gamma)\leq L_\mu-(1+C_1\gamma^2)L_n
        \leq\sum_{i=1}^n\frac{I(L_i^+;U_i)}{nC_2}.
    \end{align*}

    Furthermore, from the proof of Theorem~\ref{thm:bound-fast-rate-rademacher}, we note that the following is valid 
    \begin{align*}
    \mathrm{Err}\leq\min_{C_1,C_2>0, e^{2C_2}+e^{-2C_2(C_1\gamma^2+1)}\leq 2} C_1V(\gamma)+\sum_{i=1}^n\frac{I(L_i^+;U_i)}{nC_2}.
    \end{align*}

    Notice that the original optimization space of the variance based bound should be larger than $\{C_1,C_2|C_1,C_2>0, e^{2C_2}+e^{-2C_2(C_1\gamma^2+1)}\leq 2\}$ because in Eq.~(\ref{ineq:variance-relax}), we upper bound the most interested quantity $\mathrm{Err}-C_1V(\gamma)$ by $L_\mu-(1+C_1\gamma^2)L_n$, which restricts the original optimization space. 
    
    
    This completes the proof.
\end{proof}

\subsection{Proof of Lemma~\ref{lem:flatness-upper-bound}}

\begin{proof}
    By the definition of $\lambda$-sharpness, we notice that
    \begin{align*}
        F_i(\lambda)=&\ex{W,Z_i}{\ell(W,Z_i)-(1+\lambda)\mathbb{E}_{W|Z_i}{\ell(W,Z_i)}}^2\\
        =&\ex{Z_i}{\ex{W|Z_i}{\ell(W,Z_i)^2}-2(1+\lambda)\mathbb{E}^2_{W|Z_i}{\ell(W,Z_i)}+(1+\lambda)^2\mathbb{E}^2_{W|Z_i}{\ell(W,Z_i)}}\\
        =&\ex{W,Z_i}{\ell(W,Z_i)}-(1-\lambda^2)\ex{Z_i}{\mathbb{E}^2_{W|Z_i}{\ell(W,Z_i)}},
    \end{align*}
    where the last equality is due to the fact that the loss is the zero-one loss.
\end{proof}

\subsection{Proof of Lemma~\ref{lem:flatness-symmetric}}
\label{sec:symmetric-sharpness}
\begin{proof}
    By Lemma~\ref{lem:flatness-upper-bound}, we have
    \begin{align*}
        \mathrm{Err}-\frac{C_1}{n}\sum_{i=1}^n F_i(\lambda)=& L_\mu-(1+C_1)L_n +\frac{(1-\lambda^2)C_1}{n}\sum_{i=1}^n\ex{Z_i}{\mathbb{E}^2_{W|Z_i}{\ell(W,Z_i)}}\\
        =&\frac{1}{n}\sum_{i=1}^n\br{\ex{U_i,L_i}{L_{i,\overline{U}_i}-(1+C_1)L_{i,U_i}}+(1-\lambda^2)C_1\ex{\widetilde{Z}_{i,U_i}}{\mathbb{E}^2_{L_{i,U_i}|\widetilde{Z}_{i,U_i}}{L_{i,U_i}}}}.
    \end{align*}

    Let $\Lambda(\widetilde{Z}_{i,U_i})=\mathbb{E}^2_{L_{i,U_i}|\widetilde{Z}_{i,U_i}}{L_{i,U_i}}$ and $\Lambda(\widetilde{Z}_{i,\overline{U}_i})=\mathbb{E}^2_{L_{i,\overline{U}_i}|\widetilde{Z}_{i,\overline{U}_i}}{L_{i,\overline{U}_i}}$. Let $\Lambda(\widetilde{Z}^+_{i})=\mathbb{E}^2_{L^+_{i}|\widetilde{Z}^+_{i},U_i}L^+_{i}$ and $\Lambda(\widetilde{Z}^-_{i})=\mathbb{E}^2_{L^-_{i}|\widetilde{Z}^-_{i},U_i}L^-_{i}$.  A key observation is the following:
\begin{align*}
        &\ex{U_i,L_i}{L_{i,\overline{U}_i}-(1+C_1)L_{i,U_i}}+(1-\lambda^2)C_1\ex{\widetilde{Z}_{i,U_i}}{\Lambda(\widetilde{Z}_{i,U_i})}\\
        =& (1+\frac{C_1}{2})\ex{U_i,L_i}{L_{i,\overline{U}_i}-L_{i,U_i}}-\frac{C_1}{2}(1-\lambda^2)\ex{\widetilde{Z}_{i},U_i}{\Lambda(\widetilde{Z}_{i,\overline{U}_i})-\Lambda(\widetilde{Z}_{i,U_i})}\\
        &-\frac{C_1}{2}\br{\ex{U_i,L_i}{L_{i,\overline{U}_i}}-(1-\lambda^2)\ex{\widetilde{Z}_{i},U_i}{\Lambda(\widetilde{Z}_{i,\overline{U}_i})}}\\
        &-\frac{C_1}{2}\br{\ex{U_i,L_i}{L_{i,{U}_i}}-(1-\lambda^2)\ex{\widetilde{Z}_{i},U_i}{\Lambda(\widetilde{Z}_{i,U_i})}}\\
        =&\mathbb{E}_{L^-_{i},U_i}\left[(-1)^{U_i}\frac{C_1+2}{2}L^-_{i}-(-1)^{U_i}\frac{C_1(1-\lambda^2)}{2}\ex{\widetilde{Z}^-_{i}|U_i}{\Lambda(\widetilde{Z}^-_{i})}\right.\\
        &\left.-\frac{C_1}{2}L^-_{i}+\frac{C_1(1-\lambda^2)}{2}\ex{\widetilde{Z}^-_{i}|U_i}{\Lambda(\widetilde{Z}^-_{i})}\right]\\
        &+\mathbb{E}_{L^+_{i},U_i}\left[-(-1)^{U_i}\frac{C_1+2}{2}L^+_{i}+(-1)^{U_i}\frac{C_1(1-\lambda^2)}{2}\ex{\widetilde{Z}^+_{i}|U_i}{\Lambda(\widetilde{Z}^+_{i})}\right.\\
        &\left.-\frac{C_1}{2}L^+_{i}+\frac{C_1(1-\lambda^2)}{2}\ex{\widetilde{Z}^+_{i}|U_i}{\Lambda(\widetilde{Z}^+_{i})}\right]\\
        =&(C_1+2)\ex{L^+_{i},U_i}{(\varepsilon_i-\frac{C_1}{C_1+2})L^+_{i}-\frac{C_1(1-\lambda^2)}{C_1+2}(\varepsilon_i-1)\ex{\widetilde{Z}^+_{i}|U_i}{\Lambda(\widetilde{Z}^+_{i})}},
    \end{align*}    
    where $\varepsilon_i$ is the Rademacher variable.

    Thus,
    \begin{align}
        \mathrm{Err}-\frac{C_1}{n}\sum_{i=1}^n F_i(\lambda)=\frac{C_1+2}{n}\sum_{i=1}^n\ex{L^+_{i},U_i}{(\varepsilon_i-\frac{C_1}{C_1+2})L_{i,1}-\frac{C_1(1-\lambda^2)}{C_1+2}(\varepsilon_i-1)\ex{\widetilde{Z}^+_{i}|U_i}{\mathbb{E}^2_{L^+_{i}|\widetilde{Z}^+_{i},U_i}L^+_{i}}}.
        \label{eq:error-symmetric}
    \end{align}
    This completes the proof.
\end{proof}

\subsection{Proof of Theorem~\ref{thm:bound-flatness}}
\label{sec:proof-sharpness-bound}
\begin{proof}
    Recall Eq.~(\ref{eq:error-symmetric}),
    \[
        \mathrm{Err}-\frac{C_1}{n}\sum_{i=1}^n F_i(\lambda)=\frac{C_1+2}{n}\sum_{i=1}^n\ex{L^+_{i},U_i}{(\varepsilon_i-\frac{C_1}{C_1+2})L^+_{i}-\frac{C_1(1-\lambda^2)}{C_1+2}(\varepsilon_i-1)\ex{\widetilde{Z}^+_{i}|U_i}{\mathbb{E}^2_{L^+_{i}|\widetilde{Z}^+_{i},U_i}L^+_{i}}}.
        \]
        
    Notice that we cannot directly apply Lemma~\ref{lem:DV representation} starting from here since there is a quadratic term, namely, $\ex{\widetilde{Z}^+_{i}|U_i}{\mathbb{E}^2_{L^+_{i}|\widetilde{Z}^+_{i},U_i}L^+_{i}}$ in the RHS. 

    Inspired by \citet{yang2019fast}, we now assume that there exists a random variable $R_i$ s.t.
    \begin{align}
        &(C_1+2)\ex{L^+_{i},U_i}{(\varepsilon_i-\frac{C_1}{C_1+2})L^+_{i}-\frac{C_1(1-\lambda^2)}{C_1+2}(\varepsilon_i-1)\ex{\widetilde{Z}^+_{i}|U_i}{\mathbb{E}^2_{L^+_{i}|\widetilde{Z}^+_{i},U_i}L^+_{i}}}\notag\\
        \leq&(C_1+2)\ex{L^+_{i},U_i}{(\varepsilon_i-\frac{C_1}{C_1+2})L^+_{i}-\frac{C_1(1-\lambda^2)}{C_1+2}(\varepsilon_i-1)\ex{\widetilde{Z}^+_{i}|U_i}{R_i\mathbb{E}_{L^+_{i}|\widetilde{Z}^+_{i},U_i}L^+_{i}}}\notag\\
        =&(C_1+2)\ex{L^+_{i},U_i}{(\varepsilon_i-\frac{C_1}{C_1+2})L^+_{i}-\frac{C_1(1-\lambda^2)}{C_1+2}(\varepsilon_i-1)\ex{L^+_{i}|U_i}{R_i L^+_{i}}}\notag\\
        =&(C_1+2)\ex{L^+_{i},U_i}{\pr{(\varepsilon_i-\frac{C_1}{C_1+2})-\frac{C_1(1-\lambda^2)}{C_1+2}(\varepsilon_i-1)R_i}L^+_{i}}.
        \label{eq:error-symmetric-linear}
    \end{align}

    Such $R_i$ could satisfy $R_i\geq \sup_{\widetilde{z}^+_{i}}\mathbb{E}_{L^+_{i}|\widetilde{Z}^+_{i}=\widetilde{z}^+_{i},U_i=u_i}L^+_{i}$, for any fixed $u_i$, and the randomness of $R_i$ is controlled by $U_i$, i.e. $R_i$ is a function of $U_i$. A simple choice is to let $R_i=1$ (so $R_i$ always exists), and another choice could be letting $R_i=\ex{L^{+\prime}_{i}|U_i\sim Q_i}{L^{+\prime}_{i}}$ that satisfies the condition, where $Q_i$ is some distribution of $L^{+}_{i}$, and $L^{+\prime}_{i}$ is independent of $L^{+}_{i}$ and $\widetilde{Z}^+_{i}$ given $U_i$.

    Recall that the shifted Rademacher varaible $\tilde{\varepsilon}_i=\varepsilon_i-\frac{C_1}{C_1+2}$, and let another shifted Rademacher variable $\hat{\varepsilon}_i=\varepsilon_i-1$. Then we are ready to invoke Lemma~\ref{lem:DV representation},
    \begin{align}
        I(L^+_{i};U_i)\geq \sup_{t>0} t\ex{L^+_{i},U_i}{\pr{(C_1+2)\tilde{\varepsilon}_i-C_1(1-\lambda^2)\hat{\varepsilon}_iR_i}L^+_{i}}-\ln\ex{L^+_{i},U'_i}{e^{t\pr{(C_1+2)\tilde{\varepsilon}'_i-C_1(1-\lambda^2)\hat{\varepsilon}'_iR'_i}L^+_{i}}}.
        \label{ineq:DV-flatness}
    \end{align}

    Similar to the proof of Theorem~\ref{thm:bound-fast-rate-rademacher}, we hope the following hold
    \begin{align}
        \mathbb{E}_{L^+_{i}}\ex{U'_i}{e^{t\pr{(C_1+2)\tilde{\varepsilon}'_i-C_1(1-\lambda^2)\hat{\varepsilon}'_iR'_i}L^+_{i}}}\leq 1.
        \label{ineq:bound-moment-generating-2}
    \end{align}

    By the independence, we have
    \[
    \mathbb{E}_{L^+_{i}}\ex{U'_i}{e^{t\pr{(C_1+2)\tilde{\varepsilon}'_i-C_1(1-\lambda^2)\hat{\varepsilon}'_iR'_i}L^+_{i}}}=\ex{L^+_{i}}{\frac{e^{2t\pr{C_1(1-\lambda^2)\tilde{r}'_i-C_1-1}L^+_{i}}+e^{2tL^+_{i}}}{2}},
    \]
    where $\tilde{r}'_i\in [0,1]$ is the value (or the realization) of $R'_i$ when $U'_i=0$ (or $\varepsilon'_i=-(-1)^{U'_i}=-1$), e.g., $\tilde{r}'_i=\ex{L^{+\prime}_{i}|U'_i=0}{L^{+\prime}_{i}}$.

    Then, since $L^+_{i}$ could be either $0$ or $1$. We now consider the two cases.

    (i) When $L^+_{i}=0$, then ${\frac{e^{2t\pr{C_1(1-\lambda^2)\tilde{r}'_i-C_1-1}L^+_{i}}+e^{2tL^+_{i}}}{2}}=1$. Therefore, when $L^+_{i}=0$, the value of ${R}'_i$ has no effect on the moment generating function.

    (ii) When $L^+_{i}=1$, we have the formula  $\frac{e^{2t\pr{C_1(1-\lambda^2)\tilde{r}'_i-C_1-1}}+e^{2t}}{2}$. Notably, only when $\varepsilon'_i=-1$ (or $U'_i=0$) and $L^+_{i}=1$, the value of $R'_i$, viz, $\tilde{r}'_i$, has some impact on the moment generating function. Since $\tilde{r}'_i\in [0,1]$, it's sufficient to upper bound ${R}'_i$ by the random variable $\frac{1-\varepsilon_i'}{2}=\frac{-\hat{\varepsilon}'_i}{2}$. Thus,
    \begin{align}
        \mathbb{E}_{L^+_{i}}\ex{U'_i}{e^{t\pr{(C_1+2)\tilde{\varepsilon}'_i-C_1(1-\lambda^2)\hat{\varepsilon}'_iR'_i}L^+_{i}}}\leq \mathbb{E}_{L^+_{i}}\ex{U'_i}{e^{t\pr{(C_1+2)\tilde{\varepsilon}'_i+\frac{C_1}{2}(1-\lambda^2)\hat{\varepsilon}'^{2}_i}L^+_{i}}}.
        \label{ineq:bound-moment-generating-important}
    \end{align}

     By the moment generating function of the Bernoulli random variable $L^+_{i}$, we have
    \begin{align*}
        \mathbb{E}_{U'_i}\ex{L^+_{i}}{e^{t\pr{(C_1+2)\tilde{\varepsilon}'_i+\frac{C_1}{2}(1-\lambda^2)\hat{\varepsilon}'^{2}_i}L^+_{i}}}=&\ex{U_i}{1-\ex{L^+_{i}}{L^+_{i}}+\ex{L^+_{i}}{L^+_{i}}e^{t\pr{(C_1+2)\tilde{\varepsilon}'_i+\frac{C_1}{2}(1-\lambda^2)\hat{\varepsilon}'^{2}_i}}}\\
        =&1-\ex{L^+_{i}}{L^+_{i}}+\ex{L^+_{i}}{L^+_{i}}\frac{e^{-2t\pr{C_1\lambda^2+1}}+e^{2t}}{2}.
    \end{align*}

    Since $0\leq \ex{L^+_{i}}{L^+_{i}}\leq 1$, we only need to require that
    \[\frac{e^{-2t\pr{C_1\lambda^2+1}}+e^{2t}}{2}\leq 1.
    \]



   Replacing $t$ by $C_2$ and putting everything together (Eq.~(\ref{eq:error-symmetric}-\ref{ineq:bound-moment-generating-2})), we have
   \[
   \mathrm{Err}\leq \frac{C_1}{n}\sum_{i=1}^n F_i(\lambda)+\sum_{i=1}^n\frac{I(L^+_{i};U_i)}{nC_2}.
   \]

    This completes the proof.
\end{proof}

\subsection{Proof of Corollary~\ref{cor:bound-variance-flatness}}
\begin{proof}
    According to the proofs in Section~\ref{sec:proof-variance-bound} and Section~\ref{sec:proof-sharpness-bound}, we know that the sufficient conditions to let Eq.~(\ref{ineq:variance-bound-general}) and Eq.~(\ref{ineq:flatness-bound-general}) hold are $e^{2C_2}+e^{-2C_2(C_1\gamma^2+1)}\leq 2$ and $e^{2C_2}+e^{-2C_2(C_1\lambda^2+1)}\leq 2$, respectively. Given that both $\gamma,\lambda\in(0,1)$, there must exist $C_1,C_2$ to let both Eq.~(\ref{ineq:variance-bound-general}) and Eq.~(\ref{ineq:flatness-bound-general}) hold. Then taking minimum of these two inequalities will give us the desired result.
\end{proof}

 Additionally, in any of the following case: (i) $L_n\to 0$; (ii)$V(\gamma)\to 0$ for some $\gamma\in(0,1)$; (iii)$F(\lambda)\to 0$ for some $\lambda\in(0,1)$, we can let $C_1\to \infty$ and let $C_2=\frac{\ln{2}}{2}$, then 
$\mathrm{Err}\leq \sum_{i=1}^n\frac{2I(L_i^+;U_i)}{n\ln{2}}$. This justifies the remark after Corollary~\ref{cor:bound-variance-flatness}.

\section{Some Background on Channel Capacity of Binary Channel}
In this section, we follow the custom of the notations in \citet{thomas2006elements}, where the logarithm usually has a base of $2$ (i.e. $\log_2$). In addition, for a binary random variable, the entropy function $H(\cdot)$ can be a binary entropy function, for example, the random variable $X$ has the value $0$ and $1$, and $P(X=1)=p$, $P(X=0)=1-p$, then $H(X)=H(p)=-p\log_2{p}-(1-p)\log_2{(1-p)}$. The channel capacity of a channel between input $X$ and output $Y$ is defined as $C\triangleq\max_{P_X}I(X;Y)$.
\subsection{Channel Capacity of Binary Symmetric Channel (BSC)}
In a general case, the channel capacity of Figure~\ref{fig:binarychannel}(left) can be computed as in the following lemma.
\begin{lem}
    \label{lem:channel-capacity}
    When $X\sim P_{U_i}$, the channel capacity of the channel in Figure~\ref{fig:binarychannel}(left) is achieved and $C=(1-\alpha)\pr{1-H\pr{\frac{1-\alpha-\epsilon}{1-\alpha},\frac{\epsilon}{1-\alpha}}}$.
\end{lem}
We note that this lemma is an exercise problem in \citet[Problem~7.13]{thomas2006elements}.

\begin{proof}
    Let $P(X=0)=\pi$ and $P(X=0)=1-\pi$, then
    $
        I(X;Y)=H(Y)-H(Y|X)
        =H(Y)-H(1-\epsilon-\alpha,\alpha,\epsilon).
    $
    It's easy to see that $H(Y)=H(\pi(1-\alpha-\epsilon)+\epsilon(1-\pi),\alpha,\pi\epsilon+(1-\pi)(1-\epsilon-\alpha))$. We let $A=\pi(1-\alpha-\epsilon)+\epsilon(1-\pi)$ and $B=\pi\epsilon+(1-\pi)(1-\epsilon-\alpha))$. Notice that $A+B=1-\alpha$, then
    \begin{align*}
        H(Y)=&H(\pi(1-\alpha-\epsilon)+\epsilon(1-\pi),\alpha,\pi\epsilon+(1-\pi)(1-\epsilon-\alpha))\\
        =&-\br{(A+B)\log_2(1-\alpha)+ \alpha\log_2\alpha+A\log_2\frac{A}{1-\alpha}+ B\log_2 \frac{B}{1-\alpha}}\\
        =& H(\alpha)-(1-\alpha)\br{\frac{A}{1-\alpha}\log_2\frac{A}{1-\alpha}+ \frac{B}{1-\alpha}\log_2 \frac{B}{1-\alpha}}\\
        =& H(\alpha)+(1-\alpha)H(\frac{A}{1-\alpha},\frac{B}{1-\alpha})\leq  H(\alpha)+1-\alpha.
    \end{align*}
    To achieve the channel capacity (or to let the equality above hold), we need to let $A=B$, which indicates that $\pi=\frac{1}{2}$. 

    Thus,
    \begin{align*}
        C=H(\alpha)+1-\alpha-H(1-\epsilon-\alpha,\alpha,\epsilon)
        =&H(\alpha)+1-\alpha-\pr{H(\alpha)+(1-\alpha)H\pr{\frac{1-\epsilon-\alpha}{1-\alpha},\frac{\epsilon}{1-\alpha}}}\\
        =&(1-\alpha)\pr{1-H\pr{\frac{1-\alpha-\epsilon}{1-\alpha},\frac{\epsilon}{1-\alpha}}},
    \end{align*}
    which completes the proof.
\end{proof}

\subsection{Channel Capacity of Binary Asymmetric Channel (BAC)}
The channel capacity of the BAC channel in Figure~\ref{fig:binarychannel}(right) is given below.
\begin{lem}
\label{lem:channel-capacity-bac}
    The channel capacity of the BAC in Figure~\ref{fig:binarychannel}(right) is 
    $C=\log_2{\pr{1+\beta}}-\frac{1-q}{1-p-q}H(p)+\frac{p}{1-p-q}H(q),
    $
    where $\beta=2^{\frac{H(p)-H(q)}{1-p-q}}$ and the capacity is achived when $P(U_i=1)=\frac{1-q(1+\beta)}{(1-p-q)(1+\beta)}$. Further, if $p=0$ (i.e. Z-channel), $C=\log_2\pr{1+2^\frac{-H(q)}{1-q}}$, and for small $q$, the capacity can be approximated by $C\approx 1-\frac{1}{2}H(q)$.
\end{lem}
\begin{rem}
Notice that in this case, let $X\sim\mathrm{Bern}(1/2)$ (i.e. $X=U_i$) will not achieve the channel capacity. Thus, in the interpolating setting, except for Theorem~\ref{thm:channel-capacity-z}, we have another upper bound for $I(L^+_i;U_i)$, namely $\frac{1}{n}\sum_{i=1}^nI(L^+_i;U_i)\leq\frac{1}{n}\sum_{i=1}\ln\pr{1+2^\frac{-H(1-q_i)}{1-q_i}}$. If we further let $q_i$ be the same for each $i$ (which indeed should be true), then $I(L^+_i;U_i)\leq\ln\pr{1+2^\frac{-H(L_\mu)}{L_\mu}}$.
\end{rem}

\begin{proof}
    Let $P(X=1)=\pi$, then $I(X;Y)=H(\pi(1-p-q)+q)-\pi\pr{H(p)-H(q)}-H(q)$.
    Let $\frac{d I(X;Y)}{d \pi}=(1-p-q)\log_2\pr{\frac{1}{\pi(1-p-q)+q}-1}-H(p)+H(q)=0$, we have the optimal $\pi^*=\frac{1-q(1+\beta)}{(1-p-q)(1+\beta)}$ where $\beta=2^{\frac{H(p)-H(q)}{1-p-q}}$. Plugging $\pi=\pi^*$ into the formula of $I(X;Y)$, we have $I(X;Y)=\log_2{\pr{1+\beta}}-\frac{1-q}{1-p-q}H(p)+\frac{p}{1-p-q}H(q)$.
\end{proof}

\section{Experimental Details and Additional Results}
\label{sec:experiments-appendix}
\subsection{Experiment Setup}
\label{sec:setup}
In our linear classifier experiment, we generate synthetic Gaussian data using the widely-used Python package \emph{scikit-learn} \cite{scikit-learn}. We draw each dimension (or feature) of $X$ independently from some Gaussian distribution, and let all the features be informative to its class labels $Y$. Specifically, we choose the dimension of data $X$ to be $5$ and we create different class of points normally distributed (with the standard deviation being $1$) about vertices of an $5$-dimensional hypercube, where its sides of length can be manually controlled. In addition, we utilize full-batch gradient descent with a fixed learning rate of $0.01$ to train the linear classifier. We perform training for a total of $500$ epochs, and we employ early stopping when the training error reaches a sufficiently low threshold (e.g., $<0.5\%$). To ensure robustness and statistical significance, we draw $50$ different supersamples for each experiment. Within each supersample, we further generate $100$ different mask random variables, resulting in a total of $5,000$ runs for each experimental setting. This comprehensive setup enables us to compare both the unconditional MI bounds and the disintegrated MI bounds. Additionally, if the unconditional MI bound is the sole evaluated objective, one has the option to completely restart the training process $5,000$ times.

In the neural networks experiments, we follow the same setup with \cite{harutyunyan2021informationtheoretic,hellstrom2022a}. Specifically, we draw $k_1$ samples of $\widetilde{Z}$ and $k_2$ samples of $U$ for each given $\tilde{z}$.
For the CNN on the binary MNIST dataset, we set $k_1=5$ and $k_2=30$. The 4-layer CNN model is trained using the Adam optimizer with a learning rate of $0.001$ and a momentum coefficient of $\beta_1=0.9$. The training process spans 200 epochs, with a batch size of 128.
For ResNet-50 on CIFAR10, we set $k_1=2$ and $k_2=40$. The ResNet model is trained using stochastic gradient descent (SGD) with a learning rate of $0.01$ and a momentum coefficient of $0.9$ for a total of 40 epochs. The batch size for this experiment is set to 64. In the SGLD experiment, we once again train a 4-layer CNN on the binary MNIST dataset. The batch size is set to 100, and the training lasts for 40 epochs. The initial learning rate is $0.01$ and decays by a factor of $0.9$ after every 100 iterations. Let $t$ be the iteration index, the inverse temperature of SGLD is given by $\min\{4000, \max\{100, 10e^{t/100}\}\}$. We set the training sample size to $n=4000$, and $k_1=5$ and $k_2=30$. We save checkpoints every 4 epochs. All these experiments are conducted using NVIDIA Tesla V100 GPUs with 32 GB of memory. For more comprehensive details, including model architectures, we recommend referring to \cite{harutyunyan2021informationtheoretic,hellstrom2022a}.

Estimating the $\gamma$-variance and $\lambda$-sharpness in the CMI setting is a straightforward process. For example, to estimate sharpness, for each fixed $\tilde{z}$, we store the training losses $L_i^+$ when $U_i=0$ (corresponding to $\tilde{z}^+$) and the training losses $L_i^-$ when $U_i=1$ (corresponding to $\tilde{z}^-$) with different weight configurations $W$. By doing so, we collect the necessary data to compute the second term of the equation in Lemma~\ref{lem:flatness-upper-bound}.

\begin{figure*}[!ht]
    \centering
    \begin{subfigure}[b]{0.24\textwidth}
\includegraphics[scale=0.28]{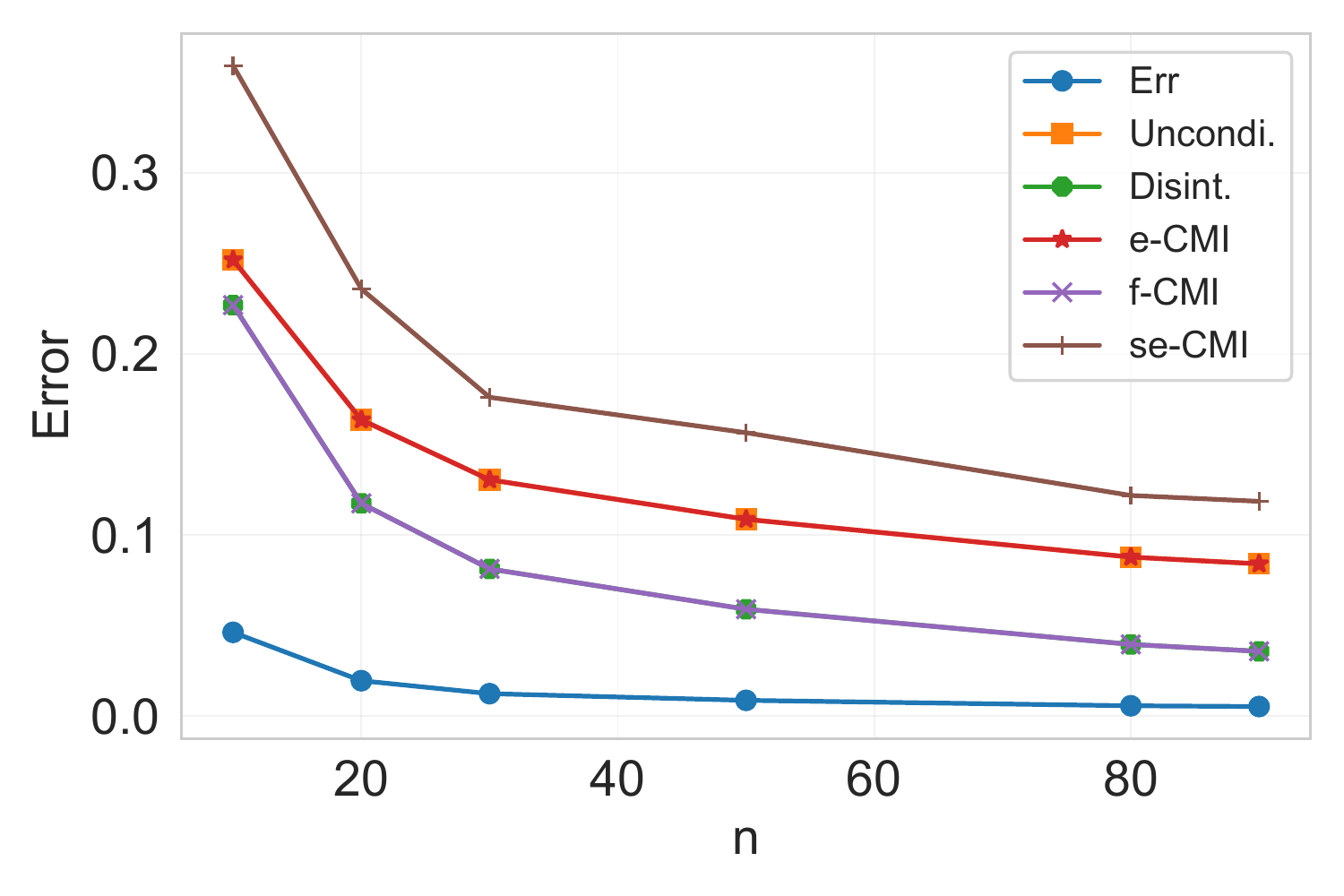}            \caption{$|\mathcal{Y}|=2$ (Realizable)}            
\label{fig:binary-square-root-easy}
    \end{subfigure}
\begin{subfigure}[b]{0.24\textwidth}
\includegraphics[scale=0.28]{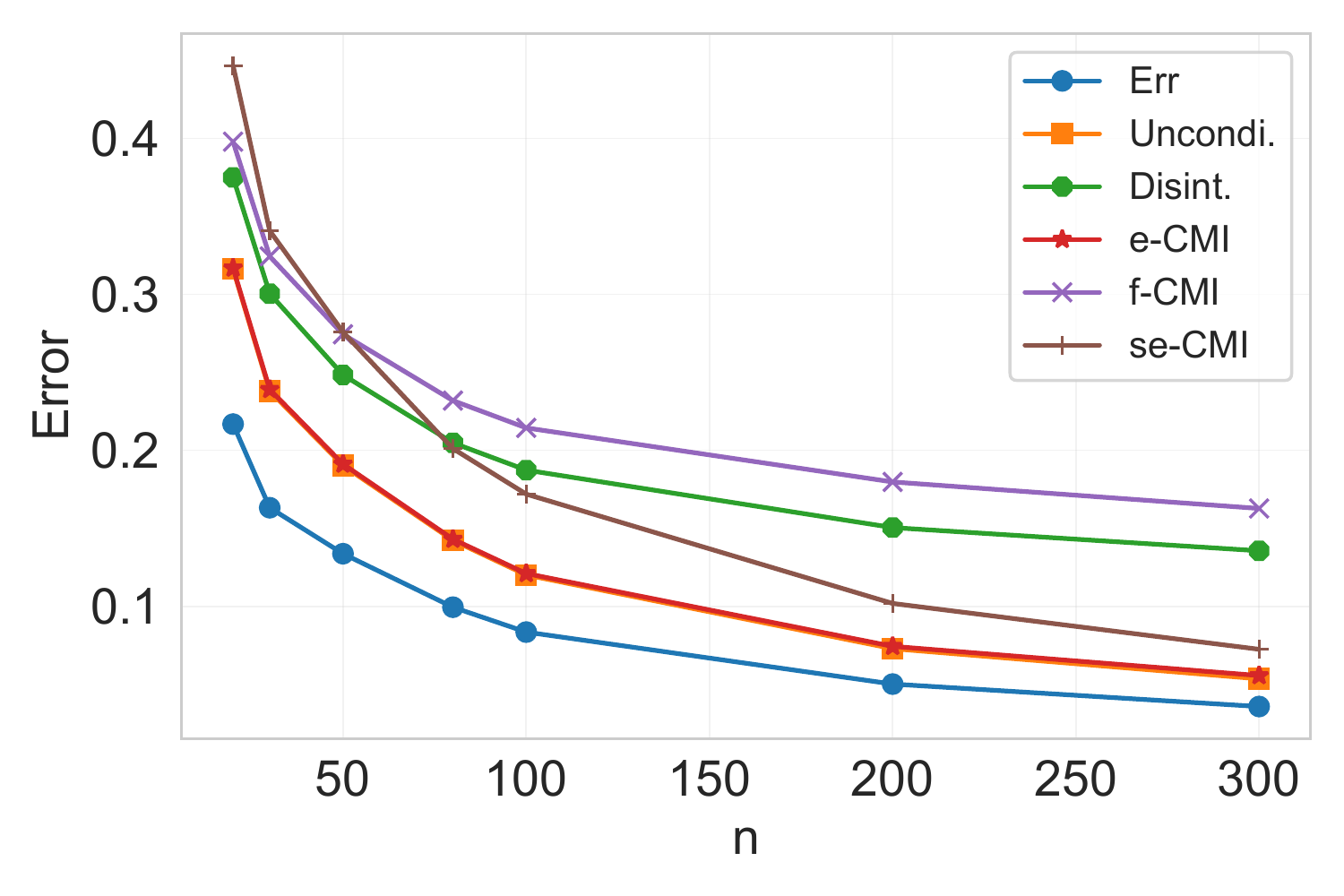}
\caption{$|\mathcal{Y}|=2$ (Non-Separable)}
    \label{fig:binary-square-root-hard}
\end{subfigure}
 \begin{subfigure}[b]{0.24\textwidth}
\includegraphics[scale=0.28]{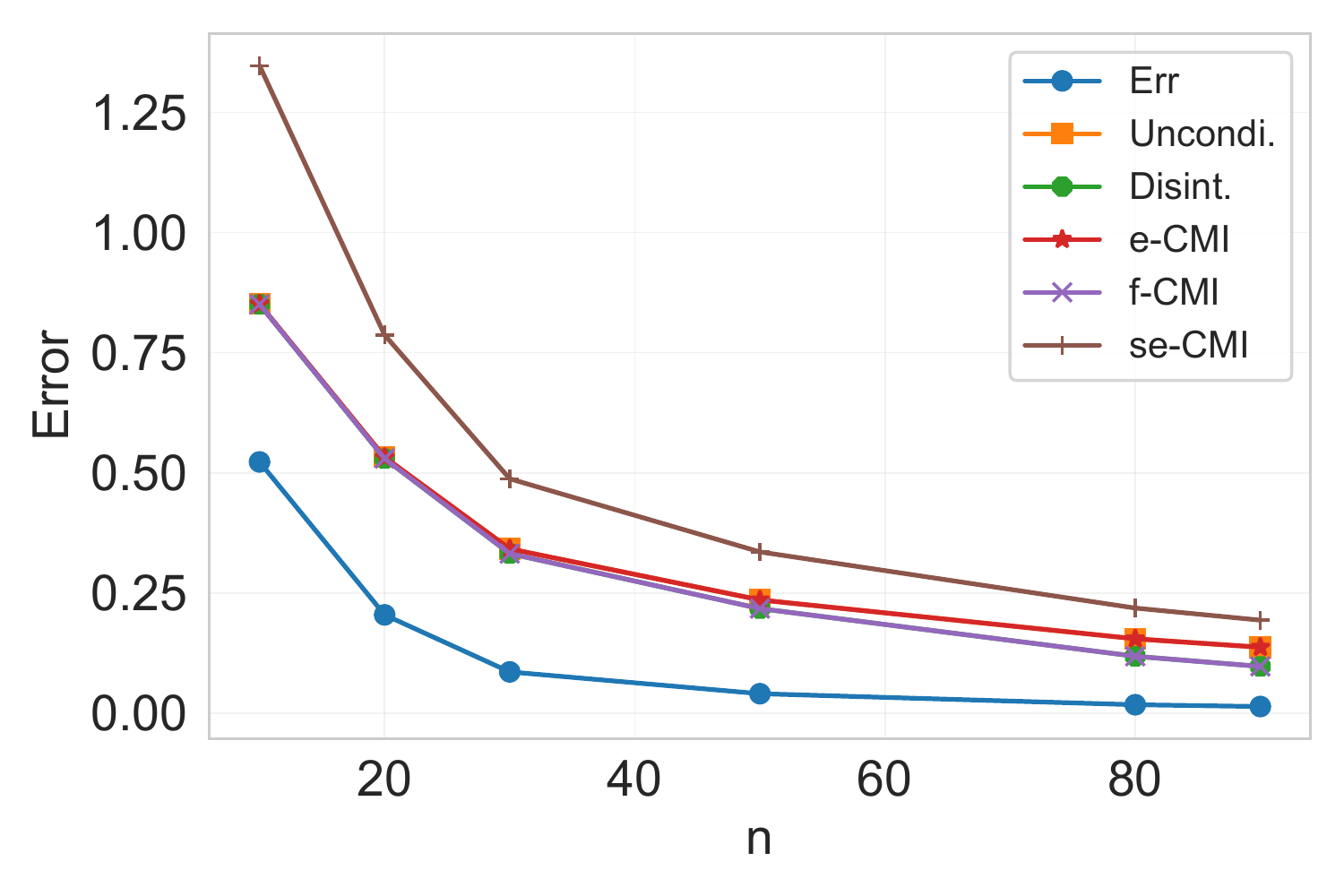}
\caption{$|\mathcal{Y}|=10$ (Realizable)}
\label{fig:ten-square-root-easy}
    \end{subfigure}
\begin{subfigure}[b]{0.24\textwidth}
\includegraphics[scale=0.28]{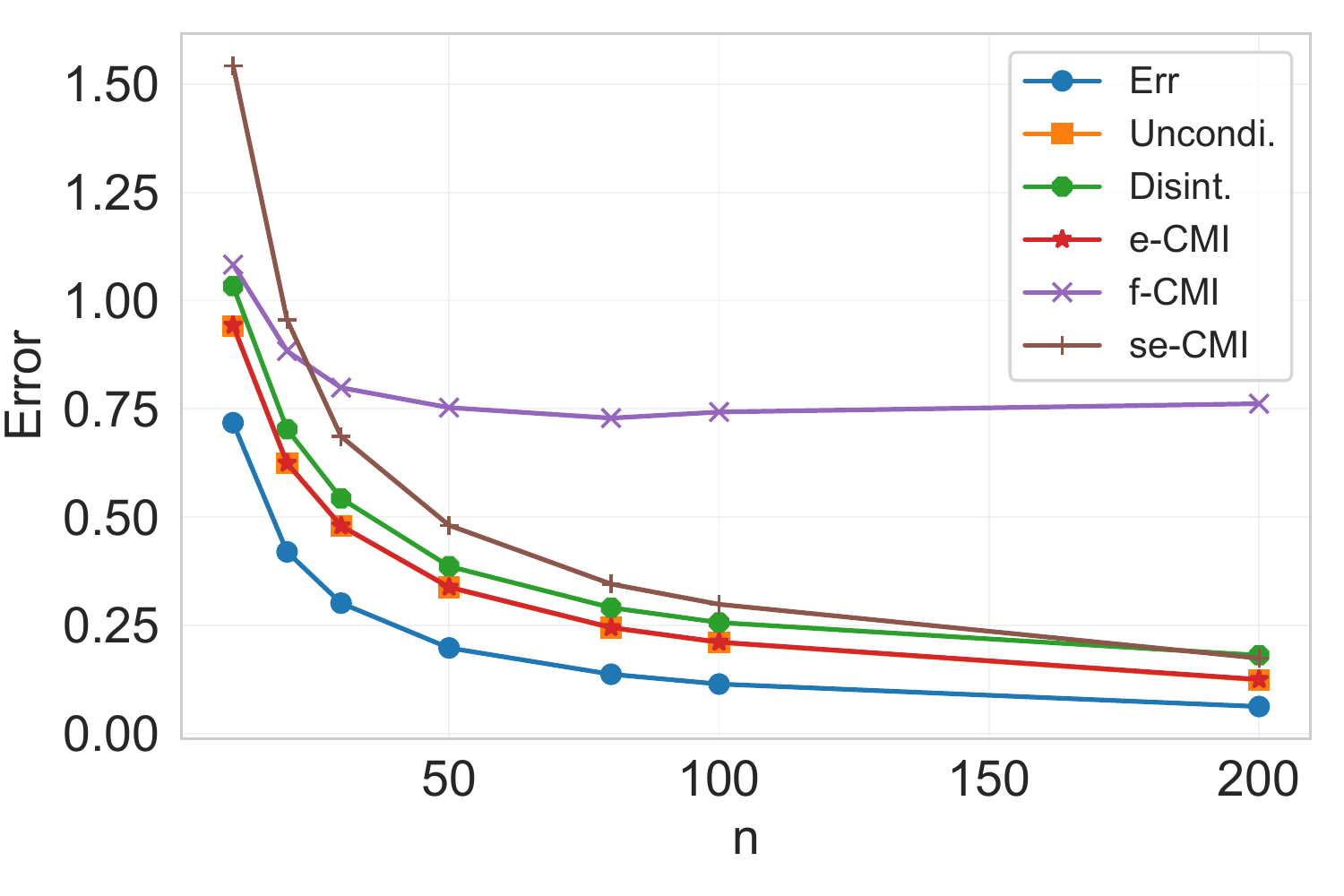}
\caption{$|\mathcal{Y}|=10$ (Non-Separable)}
\label{fig:ten-square-root-hard}
\end{subfigure}
\caption{Comparison of the square-root bounds on the synthetic dataset. Here \emph{f-CMI}, \emph{e-CMI} and \emph{se-CMI} refer to the disintegrated $f$-CMI bound \cite{harutyunyan2021informationtheoretic}, the unconditional e-CMI bound  and the single-loss square-root bound in Theorem~\ref{bound-single-loss-sq}, respectively.}\label{fig:square-root}
\end{figure*}

\subsection{Additional Numerical Results: Comparison of Square-Root Bounds}
We conduct a comparison of square-root bounds on the synthetic dataset, where we also include the disintegrated version of the $f$-CMI bound proposed by \citet{harutyunyan2021informationtheoretic}, an improved unconditional e-CMI bound (obtained by replacing $I(L_i;U_i|\widetilde{Z})$ with $I(L_i;U_i)$), and the single-loss square-root bound presented in Theorem~\ref{bound-single-loss-sq}. The results are illustrated in Figure~\ref{fig:square-root}.
Consistent with the observations in the main text, we find that the disintegrated bounds are tighter than the unconditional MI bounds when the training loss approaches zero, but looser than the unconditional MI bounds when the training loss is large. This suggests that while, according to the DPI, the unconditional e-CMI bound or ld-MI bound should be tighter than the $f$-CMI bound, in some cases, the disintegrated version of the $f$-CMI bound may be tighter than the unconditional e-CMI bound or ld-MI bound.
For non-separable $\mu$, the $f$-CMI bound becomes looser as the number of classes increases, which provides justification for the remarks made after Theorem~\ref{thm:bound-LD-cimi}. In fact, it can be even worse than the single-loss square-root bound in Theorem~\ref{bound-single-loss-sq}, which includes an undesired constant of 2.

\begin{figure*}[!ht]
    \centering
    \begin{subfigure}[b]{0.24\textwidth}
\includegraphics[scale=0.28]{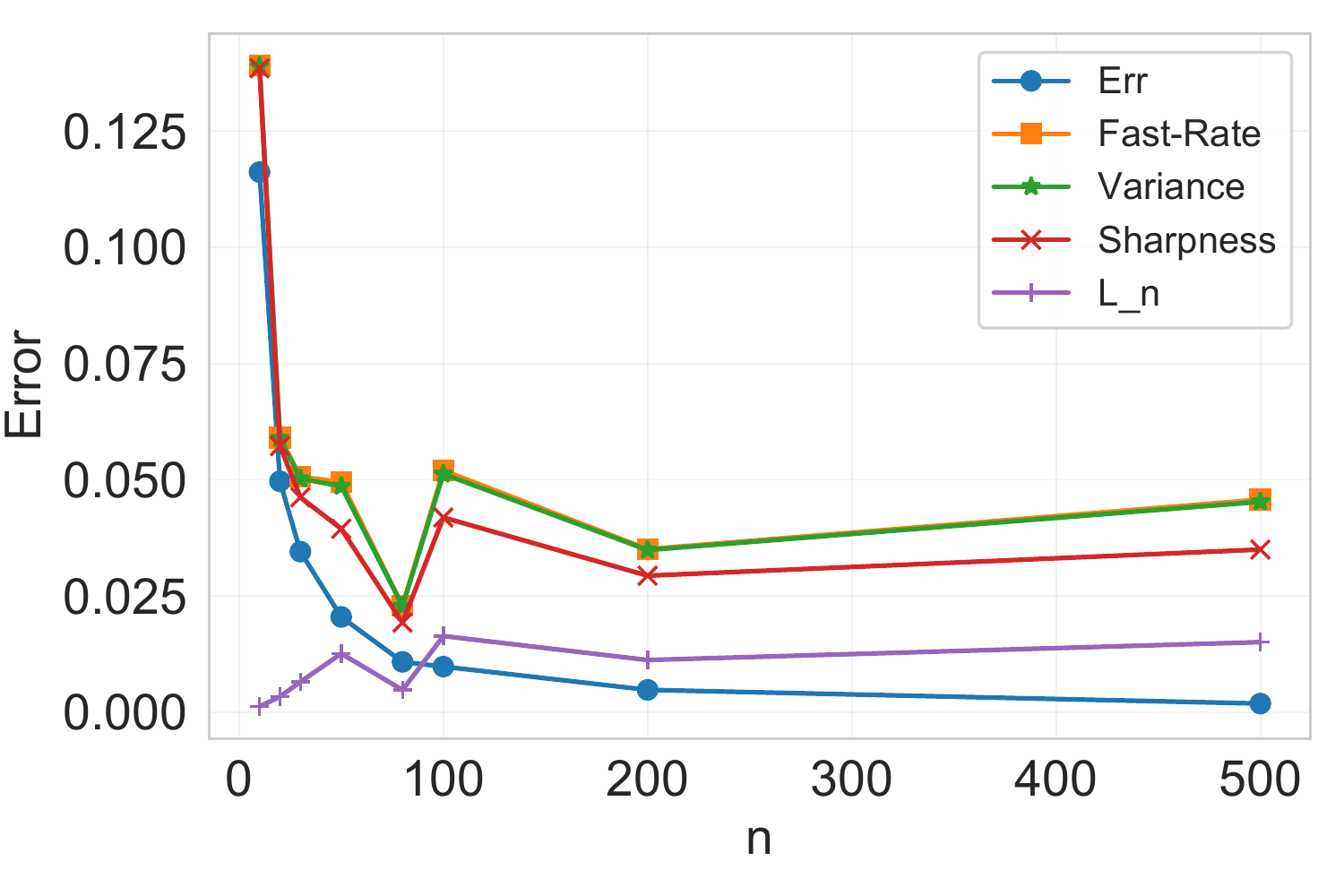}            \caption{$|\mathcal{Y}|=2$ (Small $L_n$)}            \label{fig:binary-fast-rate-easy}
    \end{subfigure}
\begin{subfigure}[b]{0.24\textwidth}
\includegraphics[scale=0.28]{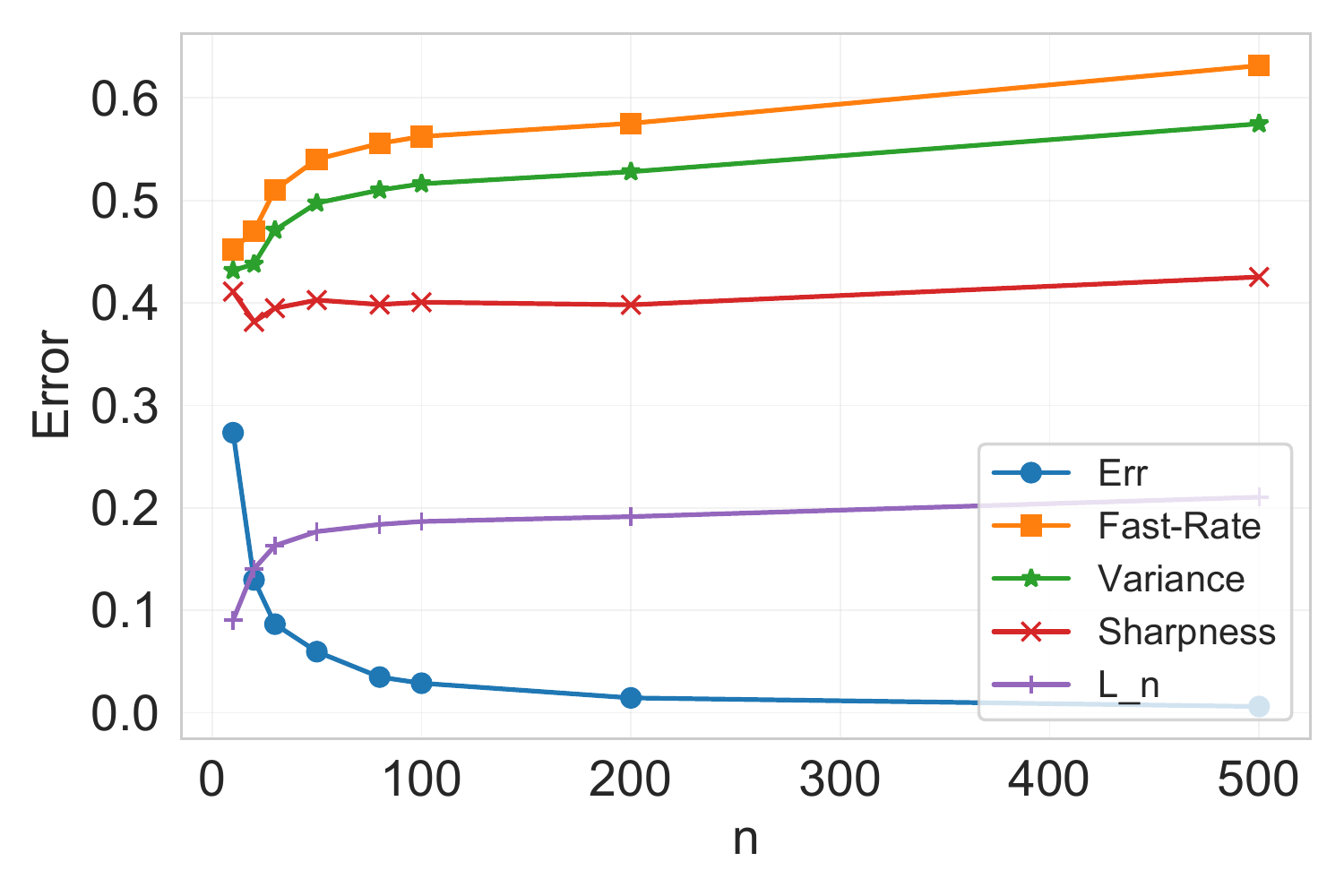}
\caption{$|\mathcal{Y}|=2$ (Large $L_n$)}
    \label{fig:binary-fast-rate-hard}
\end{subfigure}
 \begin{subfigure}[b]{0.24\textwidth}
\includegraphics[scale=0.28]{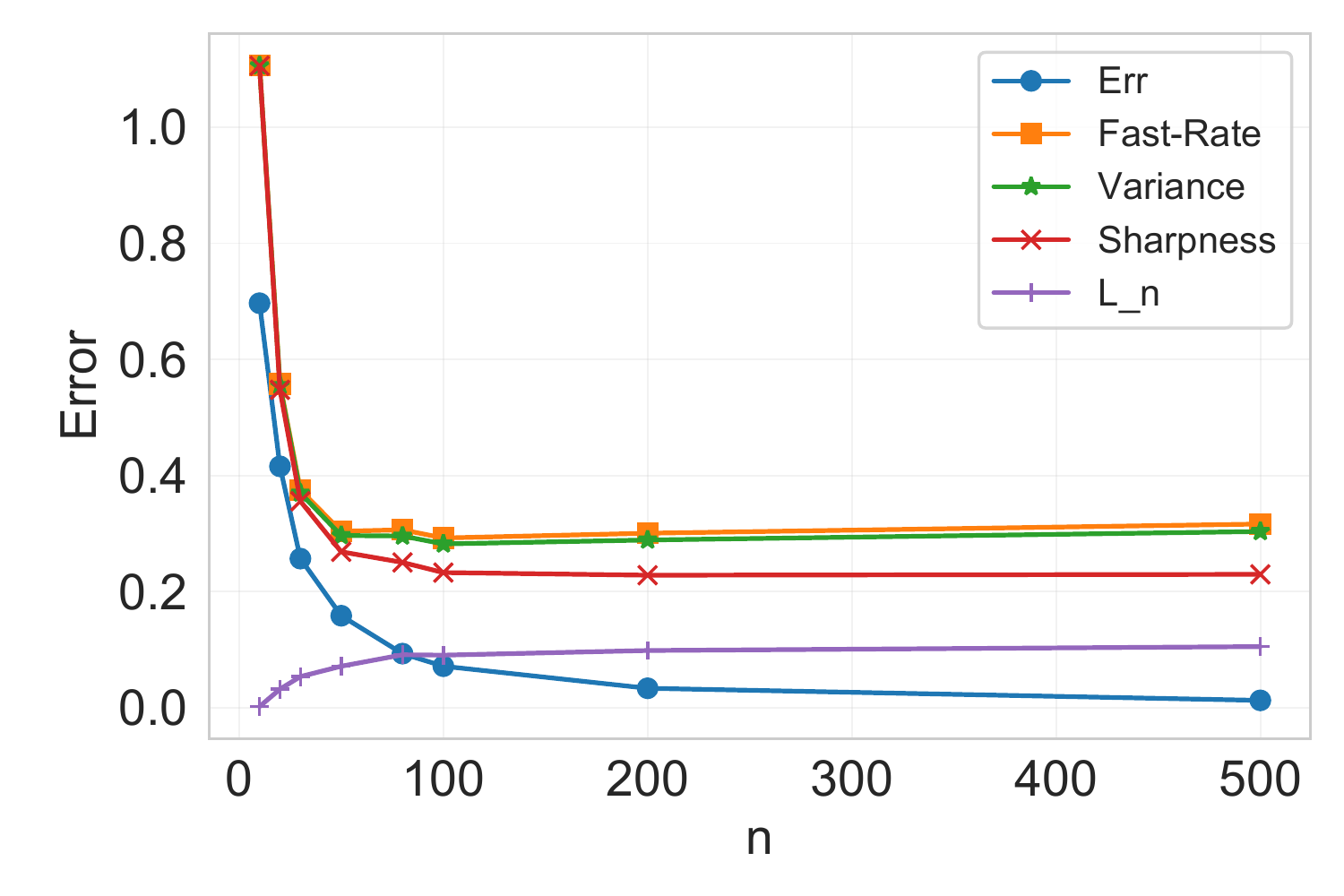}
\caption{$|\mathcal{Y}|=10$ (Small $L_n$)}
\label{fig:ten-fast-rate-easy}
    \end{subfigure}
\begin{subfigure}[b]{0.24\textwidth}
\includegraphics[scale=0.28]{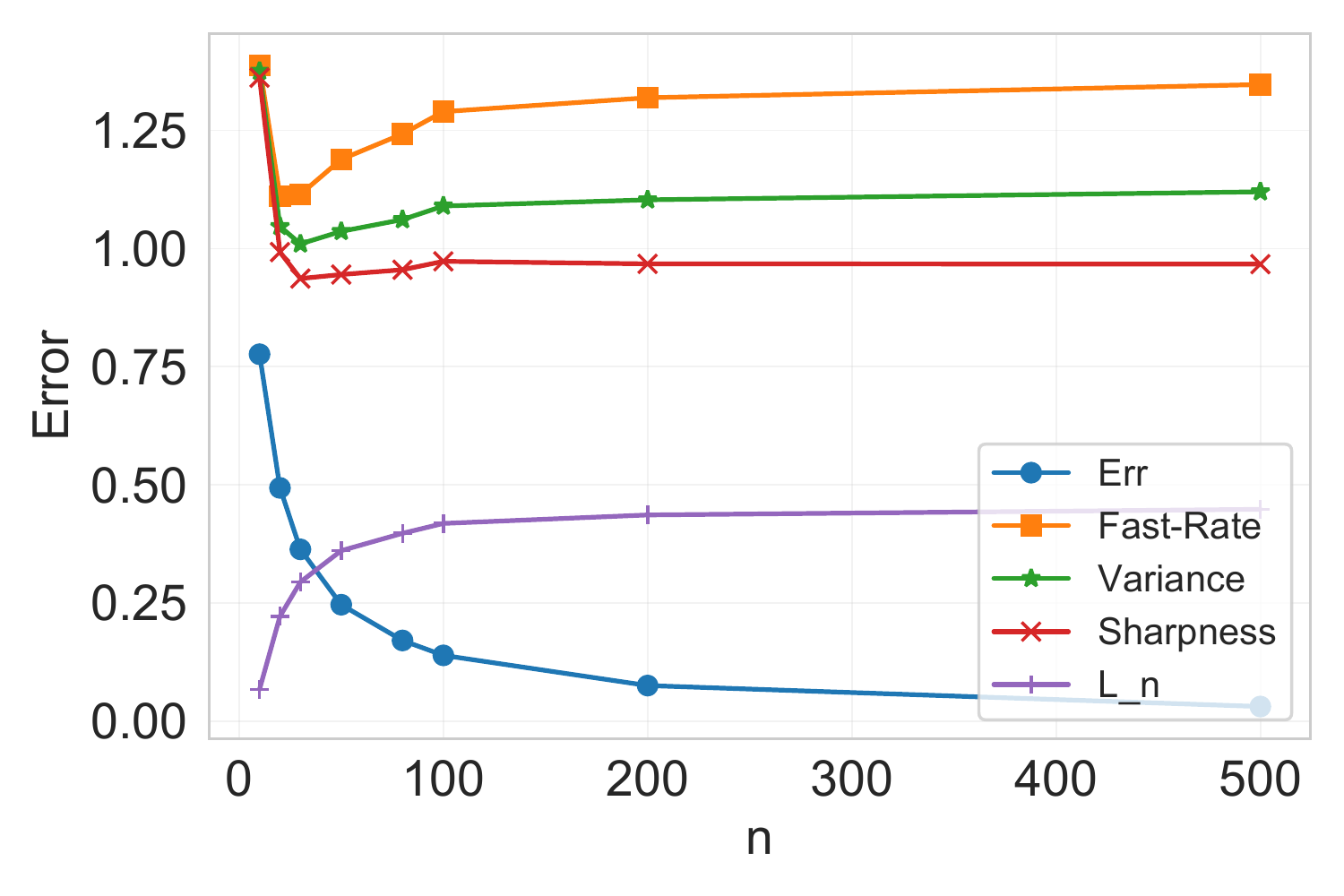}
\caption{$|\mathcal{Y}|=10$ (Large $L_n$)}
\label{fig:ten-fast-rate-hard}
\end{subfigure}
\caption{Comparison of three fast-rate bounds on the synthetic dataset. Here \emph{Fast-Rate} refers to the fast-rate bound of Eq.~(\ref{ineq:bound-fast-rate-general}) in Theorem~\ref{thm:bound-fast-rate-rademacher}.}\label{fig:fast-rate}
\end{figure*}

\subsection{Additional Numerical Results: Comparison of Fast-Rate Bounds}
We conduct a comparison of fast-rate bounds, including Eq.~(\ref{ineq:bound-fast-rate-general}) in Theorem~\ref{thm:bound-fast-rate-rademacher}, the variance bound in Theorem~\ref{thm:bound-variance}, and the sharpness bound in Theorem~\ref{thm:bound-flatness}, on the synthetic dataset with fixed values of $C_1$ and $C_2$. As mentioned in the main text, if $L_n \to 0$, both $V(\gamma)$ and $F(\lambda)$ become zero, resulting in the three bounds being equivalent. However, when $L_n \neq 0$, the variance bound and sharpness bound are always sharper than Eq.~(\ref{ineq:bound-fast-rate-general}), as discussed earlier.
In Figure~\ref{fig:fast-rate}, we compare these bounds with fixed values of $C_1=3$ and $C_2=0.3$. Figures~\ref{fig:binary-fast-rate-easy} and \ref{fig:ten-fast-rate-easy} demonstrate that when $L_n$ is small, the gap between the variance bound and Eq.~(\ref{ineq:bound-fast-rate-general}) is small, indicating that the loss variance in this case is also small. However, the sharpness bound clearly outperforms the other two bounds.
Furthermore, in Figures~\ref{fig:binary-fast-rate-hard} and \ref{fig:ten-fast-rate-hard}, when $L_n$ is large, both the sharpness bound and the variance bound significantly improve upon Eq.~(\ref{ineq:bound-fast-rate-general}). Notably, only the sharpness bound remains non-vacuous in Figure~\ref{fig:ten-fast-rate-hard}.



\end{appendices}

\end{document}